\documentclass{colt2014}

\usepackage{times}
\usepackage{graphicx}
\usepackage{natbib}
\usepackage{times}
\usepackage{amsfonts}
\usepackage{amstext}
\usepackage{latexsym}
\usepackage{color}
\usepackage{enumerate}
\usepackage{url}
\usepackage{dsfont}
\usepackage[mathscr]{euscript}
\usepackage{bm}

\makeatletter
\newtheorem*{rep@theorem}{\rep@title}
\newcommand{\newreptheorem}[2]{%
\newenvironment{rep#1}[1]{%
 \def\rep@title{#2 \ref{##1}}%
 \begin{rep@theorem}}%
 {\end{rep@theorem}}}
\makeatother

\def\Rset{\mathbb{R}}
\def\Hset{\mathbb{H}}

\DeclareMathOperator*{\E}{\rm E}

\DeclareMathOperator*{\argmin}{\rm argmin}

\DeclareMathOperator{\supp}{supp}

\DeclareMathOperator{\Tr}{Tr}

\providecommand{\norm}[2]{\lVert#1\rVert_{#2}}

\newreptheorem{theorem}{Theorem}
\newreptheorem{lemma}{Lemma}
\newreptheorem{corollary}{Corollary}
\newreptheorem{proposition}{Proposition}
\newenvironment{proof*}{\noindent{\bf Proof:}}{}

\newcommand{\h}{\widehat}
\newcommand{\wt}{\widetilde}

\renewcommand{\set}[1]{\{#1\}}
\newcommand{\tts}{\tt \small}
\newcommand{\Span}{\operatorname{span}}

\newcommand{\cA}{{\mathcal A}}
\newcommand{\cD}{{\mathcal D}}
\newcommand{\cF}{{\mathcal F}}
\newcommand{\cL}{{\mathcal L}}
\newcommand{\cQ}{{\mathcal Q}}
\newcommand{\cS}{{\mathcal S}}
\newcommand{\cT}{{\mathcal T}}
\newcommand{\cX}{{\mathcal X}}
\newcommand{\cY}{{\mathcal Y}}

\newcommand{\mat}[1]{{\mathbf #1}}

\renewcommand{\a}{\mat{a}}
\renewcommand{\b}{\mat{b}}
\newcommand{\1}{\mat{1}}
\newcommand{\bu}{\mat{u}}
\newcommand{\bv}{\mat{v}}
\newcommand{\bq}{\mat{q}}
\newcommand{\Ks}{\mat{K}_s}
\newcommand{\Kst}{\mat{K}_{st}}
\newcommand{\Kt}{\mat{K}_{t}}
\renewcommand{\v}{\mat{v}}
\newcommand{\w}{\mat{w}}
\newcommand{\x}{\mat{x}}
\newcommand{\y}{\mat{y}}
\newcommand{\z}{\mat{z}}
\newcommand{\I}{\mat{I}}
\newcommand{\X}{\mat{X}}
\newcommand{\W}{\mat{W}}
\newcommand{\Z}{\mat{Z}}

\newcommand{\qq}{{\mathsf q}}
\newcommand{\QQ}{{\mathsf Q}}

\newcommand{\UU}{{\mathsf U}}

\newcommand{\qmin}{{\qq_\text{min}}}
\newcommand{\qpmin}{{\qq'_\text{min}}}

\newcommand{\dis}{\mathrm{disc}}
\newcommand{\DIS}{\mathrm{DISC}}

\newcommand{\ignore}[1]{}

\coltauthor{\Name{Corinna Cortes} \Email{cortes@google.com} \\
       \addr Google Research \\
       76 Ninth Avenue \\
       New York, NY 10011
       \AND
       \Name{Mehryar Mohri} \Email{mohri@cims.nyu.edu} \\
       \addr Courant Institute and Google Research\\
       251 Mercer Street\\
       New York, NY 10012
       \AND 
       \Name{Andres Mu\~noz Medina} \Email{munoz@cims.nyu.edu} \\
       \addr{Courant Institute\\
       251 Mercer Street, 10012\\
       New York, NY}
}

\title[Adaptation Algorithm and Theory]{Adaptation Algorithm and Theory Based on Generalized Discrepancy}

\editor{TBD}

\begin{document} 

\maketitle

\begin{abstract} 

  We present a new algorithm for domain adaptation improving upon a
  discrepancy minimization algorithm previously shown to outperform a
  number of algorithms for this task. Unlike many previous algorithms
  for domain adaptation, our algorithm does not consist of a fixed
  reweighting of the losses over the training sample. We show that our
  algorithm benefits from a solid theoretical foundation and more
  favorable learning bounds than discrepancy minimization.  We present
  a detailed description of our algorithm and give several efficient
  solutions for solving its optimization problem. We also report the
  results of several experiments showing that it outperforms
  discrepancy minimization.
\end{abstract}

\section{Introduction}

A common problem arising in a variety of applications such as natural
language processing and computer vision is that of \emph{domain
  adaptation} \citep{Dredze07Frustratingly,
  Blitzer07Biographies,jiang-zhai07,
  LegetterWoodlang,Rosenfeld96}: quite often little or no
labeled data from the \emph{target domain} is at one's disposal, but
labeled data from a \emph{source domain} somewhat similar to the
target, as well as a relatively large amount of unlabeled data from
the target domain are available. The problem then consists of using
the source labeled and target unlabeled data, and possibly a small
amount of labeled data from the target, to derive a hypothesis
performing well on the target domain. This problem is challenging both
from the theoretical and algorithmic point of view since its scenario
does not match the standard assumption of a fixed distribution for
training and test points adopted in much of learning theory and
algorithmic design.

A theoretical analysis of the problem of adaptation has been developed
over the last few years. This includes generalization bounds based on
a notion of \emph{discrepancy}, or \emph{$d_A$-distance} in the
special case of a binary classification loss, which emerges as the
natural measure of the difference of distributions for adaptation
\citep{MansourMohriRostamizadeh2009,BenDavidBlitzerCrammerPereira2006,blitzer,CortesMohri2011}. The notion of
discrepancy has also been shown to be relevant in the analysis of the
related problem of drifting distributions \citep{drift}. Tighter
bounds than those of \citet{MansourMohriRostamizadeh2009} are given by \citet{drift} via the
use of the $\cY$-discrepancy, a finer notion of discrepancy that
depends on the labels and which therefore cannot be estimated. The
same quantity was also later used by \citet{ZhangZhangYe2012} under
the name of \emph{integral probability metric} for the analysis of
domain adaptation and multitask learning. A PAC-Bayesian study of
domain adaptation was also recently presented by
\citet{germain2013pac} based on a weighted version of the discrepancy.

Several negative results have also been given for the problem of
adaptation \citep{shai2010,BenDavidUrner2012}.  These results give
worst case lower bounds on the sample size of domain adaptation: as
stated by the authors, the problem becomes intractable when the
hypothesis set does not contain any candidate achieving a good
performance on the training set. In particular, for the counterexample
presented by \citet{BenDavidUrner2012}, the best-in-class
classification error with respect to the source distribution is only
one half.  It should be clear that adaptation can not be successful in
such cases since the only information available to the learner about
the labeling function is through the training data.

These results suggest that, as expected, adaptation cannot always be
successful. Nevertheless, there are various favorable conditions under
which an adaptation algorithm can succeed. In particular, recently, a
\emph{discrepancy minimization} (DM) algorithm was introduced by
\citet*{MansourMohriRostamizadeh2009} and further studied and enhanced by
\citet{CortesMohri2011,CortesMohri2013} which was shown both to perform
well in a number of adaptation and sample bias correction tasks and to
match or exceed the performance of several algorithms, including KLIEP
\citep{SugiyamaNakajimaKashimaVonBunauKawanabe2008}, KMM
\citep{HuangSmolaGrettonBorgwardtScholkopf2006} and a two-stage
algorithm of \citep{BickelBrucknerScheffer2009}.  In addition to its
favorable empirical performance, the DM algorithm benefits from a
series of pointwise loss guarantees for the general class of
kernel-based regularization algorithms in terms of the empirical
discrepancy and a term that depends on the closeness of the labeling
function to the hypothesis over the samples
\citep{CortesMohri2013}. One critical advantage of the DM algorithm
over previous algorithms is that the reweighting of the losses on the
training points takes into account both the loss function and the
hypothesis sets, both ignored in the design of other methods.

One shortcoming of the DM algorithm, however, is that it seeks to
reweigh the loss on the training samples to minimize a quantity
defined as the maximum over \emph{all} pairs of hypotheses, including
hypotheses that the learning algorithm might not consider as
candidates. Thus, the algorithm tends to be too conservative. We
present an alternative theoretically well founded algorithm for domain
adaptation that is based on minimizing a finer quantity, the
\emph{generalized discrepancy}, and that seeks to improve upon DM.
Unlike many previous algorithms for domain adaptation, our algorithm
does not consist of a fixed reweighting of the losses over the
training sample. Instead, the weights assigned to training sample
losses vary as a function of the hypothesis $h$. This helps us ensure
that, for every hypothesis $h$, the empirical loss on the source
distribution is as close as possible to the empirical loss on the
target distribution for that particular $h$.

We describe the learning scenario considered
(Section~\ref{sec:scenario}), then present a detailed description of
our algorithm and show that it can be formulated as a convex
optimization problem (Section~\ref{sec:algorithm}). Next, we analyze
the theoretical properties of our algorithm and show that it benefits
from more favorable learning guarantees than the DM algorithm
(Section~\ref{sec:guarantees}). This includes a study of the scenario
in which some small amount of labeled data from the target domain is
available, which may in fact be the most realistic setting for
adaptation. In Section~\ref{sec:optimization}, we analyze the
optimization problem defining our algorithm and derive an equivalent
form that can be handled by a standard convex optimization solver.  In
Section~\ref{sec:experiments}, we report the results of experiments
demonstrating that our algorithm outperforms the DM algorithm in several tasks.

\section{Learning scenario}
\label{sec:scenario}

This section defines the learning scenario of domain adaptation we
consider, which coincides with that of \citet{blitzer}, or
\citet{MansourMohriRostamizadeh2009} and \citet{CortesMohri2013}; and introduces the
definitions and concepts needed for the following sections. For the
most part, we follow the definitions and notation of 
\citet{CortesMohri2013}.

Let $\cX$ denote the input space and $\cY \subseteq \Rset$ the output
space. We define a \emph{domain} as a pair formed by a distribution
over $\cX$ and a target labeling function mapping from $\cX$ to
$\cY$. Throughout the paper, $(Q, f_Q)$ denotes the \emph{source
  domain} and $(P, f_P)$ the \emph{target domain} with $Q$ the source
and $P$ the target distribution over $\cX$ while $f_Q, f_P \colon \cX \to
\cY$, are the source and target labeling functions respectively.

In the scenario of \emph{domain adaptation} we consider, the learner
receives two samples: a labeled sample of $m$ points $\cS = ((x_1,
y_1), \ldots, (x_m, y_m)) \in (\cX \times \cY)^m$ from the source
domain with $x_1, \ldots, x_m$ drawn i.i.d.\ according to $Q$ and $y_i
= f_Q(x_i)$ for $i \in [1, m]$; and an unlabeled sample $\cT = (x'_1,
\ldots, x'_n)\in \cX^n$ of size $n$ drawn i.i.d.\ according to the
target distribution $P$. We denote by $\h Q$ the empirical
distribution corresponding to $x_1, \ldots, x_m$ and by $\h P$ the
empirical distribution corresponding to $\cT$. We will also analyze a
common scenario where, in addition to these two samples, the learner
receives a small amount of labeled data from the target domain $\cT' =
((x''_1, y''_1), \ldots, (x''_{s}, y''_{s})) \in (\cX \times
\cY)^{s}$.

We consider a loss function $L\colon \cY \times \cY \to \Rset_+$
jointly convex in its two arguments.  The $L_p$ losses commonly used
in regression and defined by $L_p(y, y') = |y' - y|^p$ for $p \geq 1$
are special instances of this definition.  For any two functions $h,
h'\colon \cX \to \cY$ and any distribution $D$ over $\cX$, we denote
by $\cL_D(h, h')$ the expected loss of $h(x)$ and $h'(x)$: $\cL_D(h,
h') = \E_{x \sim D} [L(h(x), h'(x))]$.  The learning problem consists
of selecting a hypothesis $h \in H$ out of a hypothesis set $H$ with a
small expected loss $\cL_P(h, f_P)$ with respect to the target domain.  We
further extend this notation to arbitrary functions $\qq\colon \cX \to
\Rset$ with a finite support as follows: $\cL_\qq(h, h') = \sum_{x \in
  \cX} q(x) L(h(x), h'(x))$.

\section{Algorithm}
\label{sec:algorithm}

In this section, we introduce our adaptation algorithm by first
reviewing related previous work, next presenting the key idea behind
the algorithm and deriving its general form, and finally by
formulating it as a convex optimization problem.

\subsection{Previous work}
\label{subsec:discmin}

It was shown by \citet{MansourMohriRostamizadeh2009} and
\citet{CortesMohri2011} (see also the \emph{$d_A$-distance}
\citep{BenDavidBlitzerCrammerPereira2006} in the case of binary loss
for classification) that a key measure of the difference of two
distributions in the context of adaptation is the
\emph{discrepancy}. Given a hypothesis set $H$, the discrepancy $\dis$
between two distributions $P$ and $Q$ over $\cX$ is defined by:
\begin{equation}
\dis(P, Q) = \max_{h, h' \in H} \big| \cL_{P}(h', h) - \cL_{Q}(h', h) \big|.
\end{equation}
The discrepancy has several advantages over a measure such as the
$L_1$ or total variation distance \citep{CortesMohri2013}: it is a
finer measure than the $L_1$ distance, it takes into account the loss
function and the hypothesis set, it can be accurately estimated from
finite samples for common hypothesis sets such as kernel-based ones,
it is symmetric and verifies the triangle inequality. It further
defines a distance in the case of an $L_p$ loss used with a universal
kernel such as a Gaussian kernel.

Several generalization bounds for adaptation in terms of the
discrepancy have been given in the past
\citep{BenDavidBlitzerCrammerPereira2006,MansourMohriRostamizadeh2009,CortesMohri2011,CortesMohri2013},
including pointwise guarantees in the case of kernel-based
regularization algorithms, which includes algorithms such as support
vector machines (SVM), kernel ridge regression, or support vector
regression (SVR). The bounds given in
\citep{MansourMohriRostamizadeh2009} motivated a \emph{discrepancy
  minimization} algorithm. Given a positive semi-definite (PSD) kernel
$K$, the hypothesis returned by the algorithm is the solution of the
following optimization problem
\begin{equation}
\label{eq:qmin-opt}
\min_{h \in \Hset} \quad \lambda \| h \|_K^2 + \cL_{\qmin} (h, f_Q),
\end{equation}
where $\| \cdot \|_K $ is the norm on the reproducing Hilbert space
$\Hset$ induced by the kernel $K$ and $\qmin$ is a distribution
over the support of $\h Q$ such that $\qmin = \argmin_{\qq \in \cQ}
\dis(\qq, \h P)$, where $\cQ$ is the set of all distributions defined
over the support of $\h Q$. Using $\qmin$ instead of $\h Q$ amounts to
reweighting the loss on the training samples to minimize the discrepancy
between the empirical distribution and $\h P$. Besides its theoretical
motivation, this algorithm has been shown to outperform several other
algorithms in a series of experiments carried out by \citep{CortesMohri2013}.

Observe that, by definition, the solution $\qmin$ of discrepancy
minimization is obtained by minimizing a maximum over all pairs of
hypotheses, that is $\max_{h, h' \in H} |\cL_{\h P}(h, h') -
\cL_{\qmin}(h, h')|$. But, the maximizing pair of hypotheses may not be
among the candidates considered by the learning algorithm. Thus, a
learning algorithm based on discrepancy minimization tends to be too
conservative.

\subsection{Main idea}

Assume as in several previous studies
\citep{MansourMohriRostamizadeh2009,CortesMohri2013} that the standard
algorithm selected by the learner is regularized risk minimization
over the Hilbert space $\Hset$ induced by a PSD kernel $K$. This
covers a broad family of algorithms frequently used in
applications. Ideally, that is in the absence of a domain adaptation
problem, the learner would have access to the labels of the points in
$\cT$. Therefore, he would return the hypothesis $h^*$ solution of the
optimization problem $\min_{h \in \Hset} F(h)$, where $F$ is the
convex function defined for all $h \in \Hset$ by
\begin{equation}
\label{eq:Pmin} 
F(h) = \lambda \| h \|_K^2 + \cL_{\h P} (h, f_P),
\end{equation}
where $\lambda \geq 0$ is a regularization parameter.  Thus, $h^*$ can
be viewed as the \emph{ideal hypothesis}.

In view of that, we can formulate our objective, in the \emph{presence} of a
domain adaptation problem, as that of finding a hypothesis $h$ whose
loss $\cL_P(h, f_P)$ with respect to the target domain is as close as
possible to $\cL_P(h^*, f_P)$. To do so, we will seek in fact a
hypothesis $h$ that is as close as possible to $h^*$, which would
imply the closeness of the losses with respect to the target domains.
We do not have access to $f_P$ and can only access the labels of the
training sample $\cS$. Thus, we must resort to using in our objective
function, instead of $\cL_{\h P} (h, f_P)$, a reweighted empirical
loss over the training sample $\cS$.  The main idea behind our
algorithm is to define, for any $h \in \Hset$, a reweighting function
$\QQ_h\colon \cS_\cX = \set{x_1, \ldots, x_m} \to \Rset$ such that the
objective function $G$ defined for all $h \in \Hset$ by
\begin{equation}
\label{eq:qhmin}
 G(h) = \lambda \| h \|_K^2 + \cL_{\QQ_h} (h, f_Q) 
\end{equation}
is uniformly close to $F$, thereby resulting in close
minimizers. Since the first term of \eqref{eq:Pmin} and
\eqref{eq:qhmin} coincide, the idea consists equivalently of seeking
$\QQ_h$ such that $\cL_{\QQ_h}(h, f_Q) $ and $\cL_{\h P}(h, f_P)$ be as
close as possible.  Observe that this departs from the standard
reweighting methods: instead of reweighting the training sample with
some fixed set of weights, we allow the weights to vary as a function
of the hypothesis $h$. Note that we have further relaxed the condition
commonly adopted by reweighting techniques that the weights must be
non-negative and sum to one.  Allowing the weights to be in a richer
space than the space of probabilities over $\cS_\cX$ could raise
over-fitting concerns but, we will later see that this in fact does not
affect our learning guarantees and leads to excellent empirical
results.

Of course, searching for $\QQ_h$ to directly minimize $|\cL_{\QQ_h}(h,
f_Q) - \cL_{\h P}(h, f_P)|$ is in general not possible since we do not
have access to $f_P$, but it is instructive to consider the imaginary
case where the average loss $\cL_{\h P}(h, f_P)$ is known to us for
any $h \in \Hset$. $\QQ_h$ could then be determined via
\begin{equation}
\label{eq:qh}
\QQ_h = \argmin_{\qq \in \cF(\cS_\X, \Rset)} | \cL_\qq(h, f_Q) - \cL_{\h P}(h, f_P)|,
\end{equation}
where $\cF(\cS_\X, \Rset)$ is the set of real-valued functions defined
over $\cS_\cX$. For any $h$, we can in fact select $\QQ_h$ such that
$\cL_{\QQ_h}(h, f_Q) = \cL_{\h P}(h, f_P)$ since $\cL_\qq(h, f_Q)$ is
a linear function of $\qq$ and thus the optimization problem
\eqref{eq:qh} reduces to solving a simple linear equation. With this
choice of $\QQ_h$, the objective functions $F$ and $G$ coincide and by
minimizing $G$ we can recover the ideal solution $h^*$.  Note that,
in general, the DM algorithm could not recover that ideal
solution. Even a finer discrepancy minimization algorithm exploiting
the knowledge of $ \cL_{\h P}(h, f_P)$ for all $h$ and seeking a
distribution $\qq'_\text{min}$ minimizing $\max_{h \in H} | \cL_\qq(h,
f_Q) - \cL_{\h P}(h, f_P)|$ could not, in general, recover the ideal
solution since we could not have $\cL_{\qq'_\text{min}}(h, f_Q) =
\cL_{\h P}(h, f_P)$ for all $h \in \Hset$.

Of course, in practice access to $\cL_{\h P}(h, f_P)$ is unfeasible
since the sample $\cT$ is unlabeled. Instead, we will consider a
non-empty convex set of candidate hypotheses $H'' \subseteq H$ that
could contain a good approximation of $f_P$. Using $H''$ as a set of
surrogate labeling functions leads to the following definition of
$\QQ_h$ instead of \eqref{eq:qh}:
\begin{equation}
\label{eq:agnosqh}
\QQ_h = \argmin_{\qq \in \cF(\cS_\X, \Rset)} \max_{h'' \in H''}|
\cL_\qq(h, f_Q) - \cL_{\h P}(h, h'') |.
\end{equation}
The choice of the subset $H''$ is of course key. A detailed analysis
of this choice is presented in Section~\ref{sec:guarantees}. We
present the formulation of the optimization problem for an arbitrary
choice of the convex subset $H''$.

\subsection{Formulation of optimization problem}
\label{sec:formulation-optimization}

The following result gives a more explicit expression for
$\cL_{\QQ_h}(h, f_Q)$ leading to a simpler formulation of the
optimization problem defining our algorithm.

\begin{proposition}
\label{prop:maxmin}
For any $h \in \Hset$, let $\QQ_h$ be defined by \eqref{eq:agnosqh}.
Then, the following identity holds for any $h \in \Hset$:
\begin{equation*}
  \cL_{\QQ_h}(h, f_Q) = \frac{1}{2} \Big(\max_{h'' \in H''} \cL_{\h
    P}(h, h'') + \min_{h'' \in H''}\cL_{\h P}(h, h'') \Big).
\end{equation*}
\end{proposition}

\begin{proof}
  For any $h \in \Hset$, the equation $\cL_{\bq}(h, f_Q) =
  l$ with $l \in \Rset$ admits a solution $\qq \in \cF(\cS_\X,
  \Rset)$. Thus, for any $h \in \Hset$, we can write
\begin{align*}
\cL_{\QQ_h}(h, f_Q) 
& = \argmin_{\substack{l \in \set{\cL_{\bq}(h, f_Q): \qq \in
      \cF(\cS_\X, \Rset)}}} \max_{h'' \in H''}| l - \cL_{\h P}(h, h'')| \\
& = \argmin_{l \in \Rset} \max_{h'' \in H''}| l - \cL_{\h P}(h, h'') | \\
& = \argmin_{l \in \Rset} \max_{h''\in H''} \max \Big \{ \cL_{\h P}(h, h'')
  - l, l - \cL_{\h P}(h, h'') \Big \}\\
& = \argmin_{l \in \Rset} \max \Big \{ \max_{h''\in H''} \cL_{\h P}(h, h'')
  - l, l - \min_{h''\in H'' } \cL_{\h P}(h, h'') \Big \}\\
& = \frac{1}{2} \Big(\max_{h'' \in H''} \cL_{\h
    P}(h, h'') + \min_{h'' \in H''}\cL_{\h P}(h, h'') \Big),
\end{align*}
since the minimizing $l$ is obtained for $\max_{h''\in H''} \cL_{\h P}(h,
h'') - l = l - \min_{h''\in H'' }\cL_{\h P}(h, h'')$.
\end{proof}
In view of this proposition, with our choice of $\QQ_h$ based on
\eqref{eq:agnosqh}, the objective function $G$ of our algorithm
\eqref{eq:qhmin} can be equivalently written for all $h \in \Hset$ as
follows
\begin{equation}
\label{eq:optmaxmin}
G(h) =  \lambda\| h \|_K^2 + \frac{1}{2} \Big(\max_{h'' \in
  H''} \cL_{\h P}(h, h'') + \min_{h'' \in H''}\cL_{\h P}(h, h'')\Big).
\end{equation}
The function $h \mapsto \max_{h'' \in H''} \cL_{\h P}(h, h'')$ is
convex as a pointwise maximum of the convex functions $h \mapsto
\cL_{\h P}(h, h'')$. Since the loss function $L$ is jointly convex, so
is $\cL_{\h P}$, therefore, the function derived by partial
minimization over a non-empty convex set $H''$ for one of the
arguments, $h \mapsto \min_{h'' \in H''} \cL_{\h P}(h, h'')$, also
defines a convex function \citep{BoydVandenberghe2004}. Thus, $G$ is a
convex function as a sum of convex
functions.

\section{Learning guarantees}
\label{sec:guarantees}

In this section, we present pointwise learning guarantees for our
algorithm and show that they compare favorably to the previous
guarantees given for the DM algorithm.  More formally, we prove that
there exists a family of convex sets with an element $H''$ yielding
provable better guarantees for our algorithm. Moreover, this family is
parametrized by a single variable, therefore making the search for
$H''$ tractable. As in previous work, we assume that the loss function
$L$ is \emph{$\mu$-admissible}: there exists $\mu > 0$ such that
\begin{equation}
\label{eq:mu-admissible}
|L(h(x), y) - L(h'(x), y)| \leq \mu |h(x) - h'(x)|
\end{equation}
holds for all $(x, y) \in \cX \times \cY$ and $h', h \in H$, a
condition that is somewhat weaker than $\mu$-Lipschitzness with
respect to the first argument. The $L_p$ losses commonly used in
regression, $p \geq 1$, verify this condition (see
Appendix~\ref{app:muadmissible}).

\subsection{Learning bounds and comparisons}

The existing pointwise guarantees for the DM algorithm are directly
derived from a bound on the norm of the difference of the ideal
function $h^*$ and the hypothesis obtained after reweighting the
sample losses using a distribution $\qq$. The bound is expressed in
terms of the discrepancy and a term $\eta_H(f_P, f_Q)$ measuring the
difference of the source and target labeling functions defined by
\begin{equation}
  \eta_H(f_P, f_Q) = \min_{h_0 \in H} \Big(\max_{x \in \supp(\h P)} 
  |f_P(x) - h_0(x)| + \max_{x \in \supp(\h Q)} |f_Q(x) - h_0(x)| \Big),
\end{equation}
and is given by the following proposition.

\begin{theorem}[\citep{CortesMohri2013}]
\label{th:disc} 
Let $\qq$ be an arbitrary distribution over $\cS_\cX$ and let $h^*$
and $h_\qq$ be the hypotheses minimizing $\lambda \norm{h}{K}^2 +
\cL_{\h P}(h, f_P)$ and $\lambda \norm{h}{K}^2 + \cL_{\qq}(h, f_Q)$
respectively. Then, the following inequality holds:
\begin{equation}
\label{eq:disc}
\lambda \norm{h^* - h_\qq}{K}^2 \leq \mu \, \eta_H(f_P, f_Q) + \dis(\h P, \qq). 
\end{equation}
\end{theorem}

The DM algorithm is defined by selecting the distribution $\qq$ 
minimizing the right-hand side of the bound \eqref{eq:disc}, that is
$\dis(\h P, \qq)$. We will show a result of the same nature for our
hypothesis-dependent reweighting $\QQ_h$ by showing that its choice
also coincides with that of minimizing an upper bound on $\lambda
\norm{h^* - h'}{K}^2$.

Let $\cA(H)$ be the set of all functions $\UU\colon h \mapsto \UU_h$
mapping $H$ to $\cF(\cS_\cX, \Rset)$ such that for all $h \in H$, $h
\mapsto \cL_{\UU_h}(h, f_Q)$ is a convex function. $\cA(H)$ contains
all constant functions $\UU$ such that $\UU_h = \qq$ for all $h \in
H$, where $\qq$ is a distribution over $\cS_\cX$. By
Proposition~\ref{prop:maxmin}, $\cA(H)$ also includes the function $\QQ: h
\to \QQ_h$ used by our algorithm. 

\begin{definition}[generalized discrepancy]
For any $\UU \in \cA(H)$,
we define the notion of \emph{generalized discrepancy} between $\h P$
and $\UU$ as the quantity $\DIS(\h P, \UU)$ defined by
\begin{equation}
\label{eq:DIS}
\DIS(\h P, \UU) = \max_{h \in H, h'' \in H''} |\cL_{\h P}(h, h'')  - \cL_{\UU_h}(h, f_Q) |.
\end{equation}
\end{definition}
We also denote by $d_\infty^{\h P}(f_P, H'')$ the following distance
of $f_P$ to $H''$ over the support of $\h P$:
\begin{equation}
\label{eq:delta}
d_\infty^{\h P}(f_P, H'') = \min_{h_0 \in H''} \max_{x \in \supp(\h P)} |h_0(x) - f_P(x)|.
\end{equation}
The following theorem gives an upper bound on the norm of the
difference of the minimizing hypotheses in terms of the generalized
discrepancy and $d_\infty^{\h P}(f_P, H'')$.

\begin{theorem}
\label{th:qhbound}
Let $\UU$ be an arbitrary element of $\cA(H)$ and let $h^*$ and $h_\UU$
be the hypotheses minimizing $\lambda \norm{h}{K}^2 + \cL_{\h P}(h,
f_P)$ and $\lambda \norm{h}{K}^2 + \cL_{\UU_h}(h, f_Q)$
respectively. Then, the following inequality holds for any convex set
$H'' \subseteq H$:
\begin{equation}
\label{eq:qhbound}
\lambda \norm{h^* - h_\UU}{K}^2 \leq \mu \, d_\infty^{\h P}(f_P, H'')  + 
\DIS(\h P, \UU).
\end{equation}
\end{theorem}

\begin{proof}
  Fix $\UU \in \cA(H)$ and let $G_{\h P}$ denote $h \mapsto \cL_{\h
    P}(h, f_P)$ and $G_\UU$ the function $h \mapsto \cL_{\UU_h}(h,
  f_Q)$.  Since $h \mapsto \lambda \norm{h}{K}^2 + G_{\h P}(h)$ is
  convex and differentiable and since $h^*$ is its minimizer, the
  gradient is zero at $h^*$, that is $2 \lambda h^* = -\nabla
  G_{\h P}(h^*)$. Similarly, since $h \mapsto \lambda \norm{h}{K}^2 +
  G_\UU(h)$ is convex, it admits a sub-differential at any $h \in
  \Hset$. Since $h_\UU$ is a minimizer, its sub-differential at
  $h_\UU$ must contain $0$. Thus, there exists a sub-gradient $g_0
  \in \partial G_\UU(h_\UU)$ such that $2 \lambda h_\UU = -g_0$, where
  $\partial G_\UU(h_\UU)$ denotes the sub-differential of $G_\UU$ at
  $h_\UU$. Using these two equalities we can write
\begin{align*}
  2 \lambda \norm{h^* - h_\UU}{K}^2  
& = \langle h^* - h_\UU, g_0 - \nabla G_{\h P}(h^*) \rangle\\
& = \langle g_0, h^* - h_\UU \rangle -
  \langle \nabla G_{\h P}(h^*), h^* - h_\UU \rangle \\
& \leq G_\UU(h^*) - G_\UU(h_\UU) + G_{\h P}(h_\UU) - G_{\h P}(h^*) \\
& =  \cL_{\h P}(h_\UU, f_P) -\cL_{\UU_h}(h_\UU, f_Q) + \cL_{\UU_h}(h^*, f_Q) -
\cL_{\h P}(h^*, f_P) \\
& \leq 2 \max_{h \in H} |\cL_{\h P}(h, f_P) -\cL_{\UU_h}(h, f_Q)|,
\end{align*}
where we used for the first inequality the convexity of $G_\UU$
combined with the sub-gradient property of $g_0 \in \partial
G_\UU(h_\UU)$, and the convexity of $G_{\h P}$.  For any $h \in H$,
using the $\mu$-admissibility of the loss,
we can upper bound the operand of the $\max$ operator as follows:
\begin{align*}
|\cL_{\h P}(h, h'') - \cL_{\UU_h}(h, f_Q)|
& \leq |\cL_{\h P}(h, f_P) - \cL_{\h P}(h, h_0)|  + |\cL_{\h P}(h, h_0) -
   \cL_{\UU_h}(h, f_Q)| \\
& \leq \mu \E_{x \sim \h P}|f_P(x) - h_0(x)|  +
   \max_{h'' \in H''} |\cL_{\h P}(h, h'') - \cL_{\UU_h}(h, f_Q)| \\
& \leq \mu \max_{x \in \supp(\h P)} | f_P(x) -  h_0(x) | +
  \max_{h'' \in H''} |\cL_{\h P}(h, h'') - \cL_{\UU_h}(h, f_Q)|,
\end{align*}
where $h_0$ is an arbitrary element of $H''$. Since this bound holds
for all $h_0 \in H''$, it follows immediately that
\begin{equation*}
\lambda \norm{h^* - h_\UU}{K}^2 \leq \mu \min_{h_0 \in H''} \max_{x \in \supp(\h P)} | f_P(x) -  h_0(x) | + \max_{h \in H} \max_{h'' \in H''} |\cL_{\h P}(h, h'') -
\cL_{\UU_h}(h, f_Q) |, 
\end{equation*}
which concludes the proof.
\end{proof}

Our algorithm is strongly motivated by the previous bound. Indeed, for
a fixed set $H''$, our choice of $\QQ$ precisely coincides with the
choice of $\UU_h$ minimizing the right-hand side of
\eqref{eq:qhbound}, or the second term of the bound, since the first
one does not vary with $h$ or $\UU_h$. This, however, does not imply a
better performance of our algorithm over DM. Therefore, a natural
question is whether there exists a choice of $H''$ for which
\eqref{eq:qhbound} is a uniformly tighter upper bound than
\eqref{eq:disc}.  The following proposition shows that when using an
$L_p$ loss, there exists a simple family of sets for which this
property holds.  The result is expressed in terms of the
\emph{local discrepancy} defined by:
\begin{equation*}
\label{eq:localdiscrepancy}
\dis_{H''}(\h P , \qq) = \max_{h \in H, h'' \in H''} |\cL_{\h P}(h, h'') -
\cL_{\qq}(h, h'')|,
\end{equation*}
which is a finer measure than the standard discrepancy for which the
$\max$ is defined over a pair of hypothesis \emph{both} in $H
\supseteq H''$.

\begin{theorem}
 \label{th:betterbound}
Let $L$ be the $L_p$ loss for some $p \geq 1$ and $h_0^*$ the
minimizer in the definition of $\eta_H(f_P, f_Q)$:
$h_0^* = \argmin_{h_0 \in H} \big(\max_{x \in \supp(\h P)} 
  |f_P(x) - h_0(x)| + \max_{x \in \supp(\h Q)} |f_Q(x) - h_0(x)| \big)$.
Define $r \geq 0$ by $r = \max_{x \in \supp(\h Q)} |f_Q(x) -
h_0^*(x)|$. Let $\qq$ be a distribution over $\cS_\cX$ and let $H''$
be defined by $H'' = \{h'' \in H | \cL_{\qq}(h'', f_Q)
\leq r^p\}$. Then, $h_0^* \in H''$ and the following inequality
holds:
\begin{equation}
\label{eq:boundcomp}
 \mu \, d_\infty^{\h P}(f_P, H'') + \max_{h \in H, h'' \in H''}
 |\cL_{\h P}(h, h'') - \cL_{\qq}(h, f_Q) | 
\leq \mu \, \eta_H(f_P, f_Q) + \dis_{H''}(\h P, \qq).
\end{equation}
\end{theorem}

\begin{proof}
The fact that $h_0^* \in H''$ follows from
\begin{equation*}
  \cL_{\qq}(h_0^*, f_Q) = \E_{x \sim \qq}\big[ |h_0^*(x) - f_Q(x)|^p \big] 
\leq \max_{x \in \supp(\h Q)} |h_0^*(x) - f_Q(x)|^p \leq r^p.
\end{equation*}
By Lemma~\ref{lemma:holder}, for all $h, h'' \in H$, $| \cL_{\qq}(h,
h'') - \cL_{\qq}(h, f_Q)| \leq \mu [\cL_{\qq}(h'',
f_Q)]^{\frac{1}{p}}$. In view of this inequality, we can write:
\begin{multline*}
 \max_{h \in H, h'' \in H''} |\cL_{\h P}(h, h'')
 - \cL_{\qq}(h,  f_Q) |\\ 
\begin{aligned}
 & \leq \max_{h \in H, h'' \in H''} |\cL_{\h P}(h, h'') 
 - \cL_{\qq}(h,  h'') |  + \max_{h \in H, h'' \in H''} | \cL_{\qq}(h, h'') -
  \cL_{\qq}(h, f_Q)| \\
& \leq \dis_{H''}(\h P, \qq)
 + \max_{h'' \in H''}\mu [\cL_{\qq}(h'', f_Q)]^{\frac{1}{p}} \\
& \leq \dis_{H''}(\h P, \qq) + \mu r  \\
& = \dis_{H''}(\h P, \qq) +
\mu \max_{x \in \supp(\h Q)} |f_Q(x) - h_0^*(x)|.
\end{aligned}
\end{multline*}
Using this inequality and the fact that $h_0^* \in H''$, we can write
\begin{align*}
& \mspace{-45mu} \mu \, d_\infty^{\h P}(f_P, H'') + \max_{h \in H, h'' \in H''} |\cL_{\h P}(h, h'') -
\cL_{\qq}(h, f_Q) | \\
& \leq  \mu \min_{h_0 \in H''} \max_{x \in \supp(\h P)} 
  |f_P(x) - h_0(x)| + \dis_{H''}(\h
P, \qq) + 
 \mu \max_{x \in \supp(\h Q)} |f_Q(x) - h_0^*(x)|  \\
& \leq  \mu \big(\max_{x \in \supp(\h P)} 
  |f_P(x) - h^*_0(x)| +  \max_{x \in \supp(\h Q)} |f_Q(x) - h^*_0(x)|
  \big) + \dis_{H''}(\h
P, \qq)\\
& = \mu \min_{h_0 \in H} \big(\max_{x \in \supp(\h P)} 
  |f_P(x) - h_0(x)| + \max_{x \in \supp(\h Q)} |f_Q(x) - h_0(x)| \big) + \dis_{H''}(\h
P, \qq) \\
& = \mu \, \eta_H(f_P, f_Q) + \dis_{H''}(\h P, \qq).
\end{align*}
which concludes the proof.
\end{proof}

The theorem shows that for that choice of $H''$, for any constant
function $\UU_h \in \cA(H)$ with $\UU_h = \qq$ for some fixed
distribution $\qq$ over $\cS_\cX$, the right-hand side of the bound of
Theorem~\ref{th:disc} is lower bounded by the right-hand side of the
bound of Theorem~\ref{th:qhbound}, since the local discrepancy is a
finer quantity than the discrepancy: $\dis_{H''}(\h P , \qq) \leq
\dis(\h P , \qq)$. Thus, our algorithm benefits from a more favorable
guarantee than the DM algorithm for the particular choice of $H''$, especially
since, our choice of $\QQ$ is based on the minimization over all
elements in $\cA(H)$ and not just the subset of constant functions
mapping to a distribution.

The following theorem gives pointwise guarantees for the
solution $h_\QQ$ returned by our algorithm.

\begin{corollary}
\label{coro:pointwise}
Let $h^*$ be a minimizer of $\lambda \norm{h}{K}^2 + \cL_{\h P}(h,
f_P)$ and $h_\QQ$ a minimizer of $\lambda \norm{h}{K}^2 +
\cL_{\QQ_h}(h, f_Q)$. Then, the following holds for any convex set
$H'' \subseteq H$:
\begin{equation}
\forall x \in \cX, y \in \cY, |L(h_\QQ(x), y) - L(h^*(x), y)|
\leq \mu R \sqrt{\frac{\mu \, d_\infty^{\h P}(f_P, H'')  + 
\DIS(\h P, \QQ)}{\lambda}},
\end{equation}
where $R^2 = \sup_{x \in \cX} K(x,x)$. If further $L$ is an $L_p$ loss
for some $p \geq 1$ and $H''$ defined as in
Theorem~\ref{th:betterbound}, then the following holds:
\begin{equation}
\forall x \in \cX, y \in \cY, |L(h_\QQ(x), y) - L(h^*(x), y)|
\leq \mu R \sqrt{\frac{\mu \, \eta_H(f_P, f_Q)  + 
\dis_{H''}(\h P, \qmin)}{\lambda}}.
\end{equation}
\end{corollary}

\begin{proof}
  By the $\mu$-admissibility of the loss, the reproducing property of
  $\Hset$, and the Cauchy-Schwarz inequality, the following holds for
  all $x \in \cX$ and $y \in \cY$:
\begin{multline*}
|L(h_\QQ(x), y) - L(h^*(x), y)| 
\leq \mu |h'(x) - h^*(x)| 
= |\langle h' - h^*, K(x, \cdot) \rangle_K| \\
\leq \|h' - h^*\|_K \sqrt{K(x, x)} \leq R \|h' - h^*\|_K.
\end{multline*}
Upper bounding $\|h' - h^*\|_K$ using the bound of
Theorem~\ref{th:qhbound} and using the fact that $\QQ$ is
a minimizer of the bound over all choices of $\UU \in \cA(H)$ yields
the desired result.
\end{proof}
The pointwise loss guarantees just presented can be directly 
used to bound the difference of the expected loss of $h^*$ and
$h_\QQ$ in terms of the same upper bounds, e.g.,
\begin{equation}
\label{eq:genbound}
\cL_P(h_\QQ, f_P) \leq \cL_P(h^*, f_P)| +
\mu R \sqrt{\frac{\mu \, d_\infty^{\h P}(f_P, H'')  + 
\DIS(\h P, \QQ)}{\lambda}}.
\end{equation}
The results presented in this section suggest selecting $H''$ to
minimize the right-hand side of \eqref{eq:genbound}. The space over
which $H''$ is searched is the family of all balls centered in $f_Q$
defined in terms of $\cL_{\qmin}$, which is parametrized only by the
radius $r$. This is motivated by Theorem~\ref{th:betterbound} which
shows that this family contains choices for $H''$ with provably more
favorable guarantees than that of the DM algorithm. Given a small
amount of labeled data from the target domain (which is often the case
in practice), it can be used as a validation set to select the value
of r minimizing the bound of Corollary~\ref{coro:pointwise}.

\subsection{Scenario of additional labeled data}

Here, we consider a rather common scenario in practice where, in
addition to the labeled sample $\cS$ drawn from the source domain and the
unlabeled sample $\cT$ from the target domain, the learner receives a
small amount of labeled data from the target domain $\cT' = ((x''_1,
y''_1), \ldots, (x''_{s}, y''_{s})) \in (\cX \times \cY)^{s}$. This
sample is typically too small to be used solely to train an
algorithm and achieve a good performance. However, it can be
useful in at least two ways that we discuss here.

One important benefit of $\cT'$ is to serve as a validation set to
determine the parameter $r$ that defines the convex set $H''$ used by
our algorithm.  Another use of $\cT'$ is to augment
our algorithm to exploit the additional source of information it
provides.  Our learning guarantees can be extended to cover this
case. Let $\h P'$ denote the empirical distribution associated to
$\cT'$. To take advantage of $\cT'$, our algorithm can be trained on
the sample of size $(m + s)$ obtained by combining $\cS$ and $\cT'$,
which corresponds to the new empirical distribution $\h Q' =
\frac{m}{m + s} \h Q + \frac{s}{m + s} \h P'$. Note that for large
values of $s$, $\h Q'$ essentially ignores the points from the source
distribution $Q$, which corresponds to the standard supervised
learning scenario in the absence of adaptation. Let $\qpmin$ denote
the discrepancy minimization solution when using $\h Q'$. Since
$\supp(\h Q') \supseteq \supp(\h Q)$, the local discrepancy using
$\qpmin$ is a lower bound on the local discrepancy using $\qmin$:
\begin{equation*}
 \dis_{H''}(\qpmin, \h P) 
= \min_{\supp(\qq) \subseteq \supp(\h Q')} \dis_{H''}(\h P, \qq) 
\leq \min_{\supp(\qq) \subseteq \supp(\h Q)} \dis_{H''}(\h P, \qq) = \dis_{H''}(\qmin, \h P).
\end{equation*}
Thus, in view of Corollary~\ref{coro:pointwise}, for an appropriate
choice of $H''$, the learning guarantee for our algorithm is more
favorable when using $\h Q'$, which suggests that, using the limited
amount of labeled points from the target distribution can improve the
performance of our algorithm.

\section{Optimization solution}
\label{sec:optimization}

As shown in Section~\ref{sec:formulation-optimization}, the function
$G$ defining our algorithm is convex and the problem of minimizing the
expression \eqref{eq:optmaxmin} is a convex optimization problem.
Nevertheless, the problem is not straightforward to solve, in
particular because evaluating the term $\max_{h'' \in H''} \cL_{\h
  P}(h, h'')$ that it contains requires solving a non-convex
optimization problem. We present two solutions for the problem in the
case of the $L_2$ loss: an exact solution obtained by solving a
semi-definite programming (SDP) problem, which we prove is equivalent
to the original optimization problem in the case of a broad family of
convex sets $H''$; and an approximate solution for an arbitrary convex
set $H''$ based on sampling and solving a quadratic programming (QP)
problem.

\subsection{SDP formulation}
\label{sec:sdp}

As discussed in Section~\ref{sec:guarantees}, the choice of $H''$ is a
key component of our algorithm. In view of
Corollary~\ref{coro:pointwise}, we will consider the set $ H' = \set{
h'' \,|\, \cL_{\qmin}(h'', f_Q) \leq r^2 }$, for a fixed value of
$r$. Define $W$ by $W = \Span(K(x_1, \cdot), \ldots, K(x_m, \cdot))$
and denote by $W^\bot$ its orthogonal complement. By the reproducing
property of $\Hset$, for every $h^\bot \in W^\bot$ we have
$h^\bot(x_i) = \langle h^\bot, K(x_i, \cdot) \rangle_K = 0$. Thus, the
equality $\cL_{\qmin}(h'', f_Q) = \cL_{\qmin}(h'' + h^\bot, f_Q)$
holds for for any function $h''$. We will therefore consider only
hypotheses in the subspace $W$ and define $H''$ to be equal to the set
$\{ \a \in \Rset^m | \sum_{j=1}^m \qmin(x_j)(\sum_{i=1}^m a_i
\qmin(x_i)^{1/2}K(x_i, x_j) - y_j)^2 \leq r^2\}$. Similarly, by the
representer theorem, we know the solution to \eqref{eq:optmaxmin} will
be of the form $h = n^{-1/2}\sum_{i=1}^nb_i K(x_i', \cdot)$.

We define the \emph{normalized} kernel matrices $\Kt$, $\Ks$, and
$\Kst$ respectively by $\Kt^{ij} = n^{-1} K(x_i', x_j')$, $\Ks^{ij} =
\qmin(x_i)^{1/2} \qmin(x_j)^{1/2} K(x_i, x_j)$ and $\Kst^{ij} =
n^{-1/2} \qmin(x_j)^{1/2} K(x_i', x_j)$. For our choice of the convex set
$H''$, problem \eqref{eq:optmaxmin} is then equivalent to
\begin{equation}
\label{eq:kmaxmin}
\min_{\b \in \Rset^n} \lambda \b^\top \Kt \b +
\frac{1}{2} \left(\max_{\substack{\a \in \Rset^m \\ \|\Ks \a - \y\|^2 \leq r^2}} \|\Kst \a - \Kt \b\|^2 
+ \min_{\substack{\a \in \Rset^m \\ \|\Ks \a - \y\|^2 \leq r^2}} \|\Kst \a - \Kt \b\|^2 \right),
\end{equation}
where $\y = (\qmin(x_1)^{1/2}y_1, \ldots, \qmin(x_m)^{1/2} y_m)$ is
the vector of normalized labels.

\begin{lemma}
\label{lemma:sdpdual} 
The Lagrangian dual of the problem
%\label{eq:maxprob}
$\max_{\substack{\a \in \Rset^m \\ \|\Ks \a - \y\|^2 \leq r^2}} 
 \ \frac{1}{2}\|\Kst \a\|^2 - \b^\top \Kt \Kst \b$
is given by
\vspace{-.75cm}
\begin{align*}
\min_{\eta \geq 0, \gamma} & \ \gamma \\  
\text{s. t.} & \ \left( 
\def\arraystretch{1.3}
\begin{array}{cc} 
 -\frac{1}{2} \Kst^\top \Kst + \eta \Ks^2  
& \frac{1}{2}\Kst^\top \Kt \b - \eta \Ks\y  \\
\frac{1}{2} \b^\top  \Kt \Kst - \eta \y^\top \Ks  
&  \eta (\|\y\|^2 - r^2) + \gamma 
\end{array}
\right) \succeq  0. 
\end{align*}
Furthermore, the duality gap for these problems is zero.
\end{lemma}
The proof of the lemma is given in Appendix~\ref{app:sdpdual}. The
lemma helps us derive the following equivalent SDP formulation for our
original optimization problem. Its solution can be found in
polynomial time using standard convex optimization solvers.

\begin{proposition}
\label{prop:cone}
The optimization problem \eqref{eq:kmaxmin} is equivalent to the
following SDP:
\begin{align*}
\max_{\alpha, \beta, \nu, \Z, \z} & \ \frac{1}{2} \Tr(\Kst^\top \Kst \Z)
- \beta - \alpha\\
\text{s. t} & \ \left(
\def\arraystretch{1.3}
\begin{array}{cc}
\nu \Ks^2 + \frac{1}{2}\Kst^\top \Kst - \frac{1}{4} \wt{\mat K}
& \nu  \Ks \y + \frac{1}{4}\wt{\mat K} \z \\
\nu  \y^\top \Ks + \frac{1}{4} \z^\top \wt{\mat K}
&\alpha + \nu (\|\y\|^2 - r^2)
\end{array}
\right) \succeq 0 \quad \wedge \quad
\left(
\begin{array}{cc}
\Z & \z \\
\z^\top & 1
\end{array}
\right) \succeq 0 \\
& \ \left(
\def\arraystretch{1.3}
\begin{array}{cc}
\lambda \Kt + \Kt^2 & \frac{1}{2} \Kt \Kst \z \\
\frac{1}{2} \z^\top \Kst^\top \Kt & \beta
\end{array}
\right) \succeq 0
\quad \wedge \quad \Tr(\Ks^2 \Z) - 2\y^\top \Ks \z + \|\y\|^2 \leq r^2
\quad \wedge \quad \nu \geq 0,
\end{align*}
where $\wt{\mat K} = \Kst^\top \Kt (\lambda \Kt + \Kt^2)^\dag \Kt \Kst$.
\end{proposition}

\subsection{QP formulation}
\label{sec:optimizationqp}

The SDP formulation described in the previous section is applicable
for a specific choice of $H''$. In this section, we present an
analysis that holds for an arbitrary convex set $H''$. First, notice
that the problem of minimizing $G$ (expression \eqref{eq:optmaxmin}) is
related to the minimum enclosing ball (MEB) problem.  For a set $D
\subseteq \Rset^d$, the MEB problem is defined as follows:
\begin{equation*}
  \min_{\bu \in \Rset^d}\max_{\bv \in D} \|\bu - \bv\|^2.
\end{equation*}
Omitting the regularization and the $\min$ term from
\eqref{eq:optmaxmin} leads to a problem similar to the MEB.  Thus, we
could benefit from the extensive literature and algorithmic study
available for this problem
\citep{welz1991,Kumar03,schonherr,fischer2003fast,Yildirim2008}.
However, to the best of our knowledge, there is currently no
solution available to this problem in the case of an infinite
set $D$, as in the case of our problem.
Instead, we present a solution for solving an approximation of
\eqref{eq:optmaxmin} based on sampling.

Let $\{h_1, \ldots, h_k\}$ be a set of hypotheses in $\partial H''$
and let $\mathcal{C} = \mathcal{C}(h_1, \ldots, h_k)$ denote their
convex hull. The following is the sampling-based approximation of
\eqref{eq:optmaxmin} that we consider:
\begin{equation}
\label{eq:optmaxapp}
\min_{h \in \Hset} \lambda \norm{h}{K}^2 +
 \frac{1}{2} \max_{i=1,...,k} \cL_{\h P}(h, h_i) + \frac{1}{2}
 \min_{h' \in  \mathcal{C}} \cL_{\h P }(h, h').
\end{equation}

\begin{proposition} 
\label{prop:dual}
Let $\mat Y =(Y_{ij}) \in \Rset^{n \times k}$ be the
matrix defined by $Y_{ij} = n^{-1/2} h_j(x_i')$ and $\y' = (y'_1,
\ldots, y'_k)^\top \in \Rset^k $ the vector defined by $y'_i = n^{-1}
\sum_{j=1}^n h_i(x'_j)^2$. Then, the dual problem of
\eqref{eq:optmaxapp} is given by
\begin{align}
\label{eq:dualapp}
\max_{\bm \alpha, \bm \gamma, \beta} & \ -\Big(\mat Y \bm \alpha + \frac{\bm
  \gamma}{2} \Big)^\top \Kt\Big(\lambda \I + \frac{1}{2}\Kt\Big)^{-1}
\Big(\mat Y \bm \alpha  + \frac{\bm \gamma}{2}\Big) - \frac{1}{2} \bm
\gamma^\top \Kt \Kt^\dag \bm \gamma +  \bm \alpha^\top \y' - \beta\\
\text{s.t.} & \ \1^\top \bm \alpha = \frac{1}{2}, \qquad  \1
\beta \geq -\mat Y^\top \bm \gamma, \qquad \bm \alpha\geq
0, \nonumber
\end{align}
where $\1$ is the vector in $\Rset^k$ with all components equal to
$1$. Furthermore, the solution $h$ of \eqref{eq:optmaxapp} can be
recovered from a solution $(\bm \alpha, \bm \gamma, \beta)$ of
\eqref{eq:dualapp} by $\forall x, h(x) =\sum_{i = 1}^n a_i K(x_i, x)$,
where $\bm a = \big(\lambda \I + \frac{1}{2}\Kt)^{-1}(\mat Y \bm
\alpha + \frac{1}{2}\bm \gamma)$.
\end{proposition}

The proof of the proposition is given in Appendix~\ref{app:qpformula}.
The result shows that, given a finite sample $h_1, \ldots, h_k$ on the
boundary of $H''$, \eqref{eq:optmaxapp} is in fact equivalent to a
standard QP. Hence, a solution can be found efficiently with one of
the many off-the-shelf algorithms for quadratic programming.

We now describe the process of sampling from the boundary of the set
$H''$, which is a necessary step for defining problem
\eqref{eq:optmaxapp}. We consider compact sets of the form $H'':=
\{h'' \in \Hset \; | \; g_i(h'') \leq 0\}$, where the functions $g_i$
are continuous and convex. For instance, we could consider the set
$H''$ defined in the previous section. More generally, we can consider
a family of sets $H''_p = \{h'' \in H | \; | \; \sum_{i=1}^m
\qmin(x_i)|h(x_i) -y_i|^p \leq r^p\}$.
 
Assume that there exists $h_0$ satisfying $g_i(h_0) < 0$. Our sampling
process is illustrated by Figure~\ref{fig:hsampling} and works as follows: pick a
random direction $\h h$ and define $\lambda_i$ to be the minimal
solution to the system
\begin{equation*}
  (\lambda \geq 0) \wedge (g_i(h_0 + \lambda \h{h}) = 0).
\end{equation*}
Set $\lambda_i = \infty$ if no solution is found and define $\lambda^*
= \min_i \lambda_i$. Notice that the compactness of $H''$ guarantees
the condition $\lambda^* < \infty$. The hypothesis $h = h_0 +
\lambda^* \h h$ satisfies $h \in H''$ and $g_j(h) = 0$ for $j$ such
that $\lambda_j = \lambda^*$. The latter is straightforward. To verify
the former, assume that $g_i(h_0 + \lambda^* \h h) > 0$ for some
$i$. The continuity of $g_i$ would imply the existence of $\lambda_i'$
with $0 < \lambda'_i < \lambda^* \leq \lambda _i$ such that $g_i(h_0 +
\lambda_i' \h h) = 0$. This would contradict the choice of
$\lambda_i$, thus, the inequality $g_i(h_0 + \lambda^* \h h) \leq0$
must hold for all $i$.

Since a point $h_0$ with $g_i(h_0) < 0$ can be obtained by solving a
convex program and solving the equations defining $\lambda_i$ is, in
general, simple, the process described provides an efficient way of
sampling points from the convex set $H''$.

\begin{figure}[t]
\centering
\includegraphics[scale=.4]{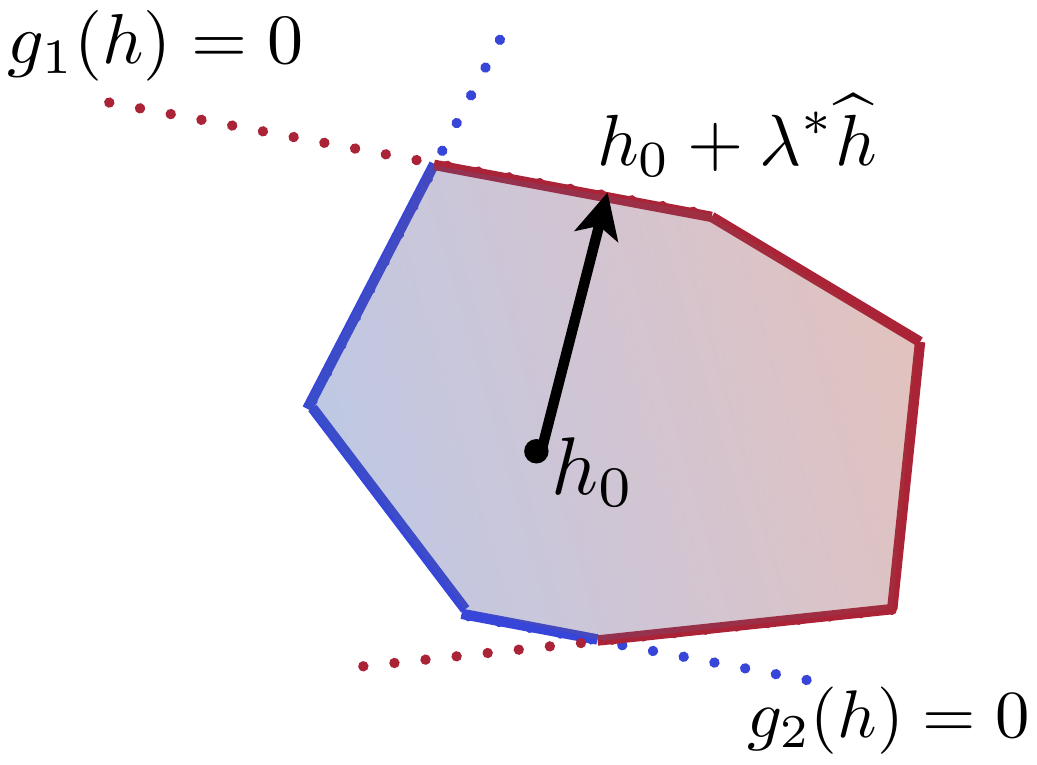}
\vskip -.4cm
\caption{Illustration of the sampling process on the set $H''$.}
\label{fig:hsampling}
\end{figure}

In the next section, we report the results of 
our experiments with our algorithm in several tasks in which
it outperforms the DM algorithm.

\section{Experiments}
\label{sec:experiments}
The results of extensive comparisons between GDM and several other
adaptation algorithms is presented in this section with favorable
results for our algorithm. 

\begin{figure}[t]
\centering
\begin{tabular}{cc}
(a) \includegraphics[scale=.32]{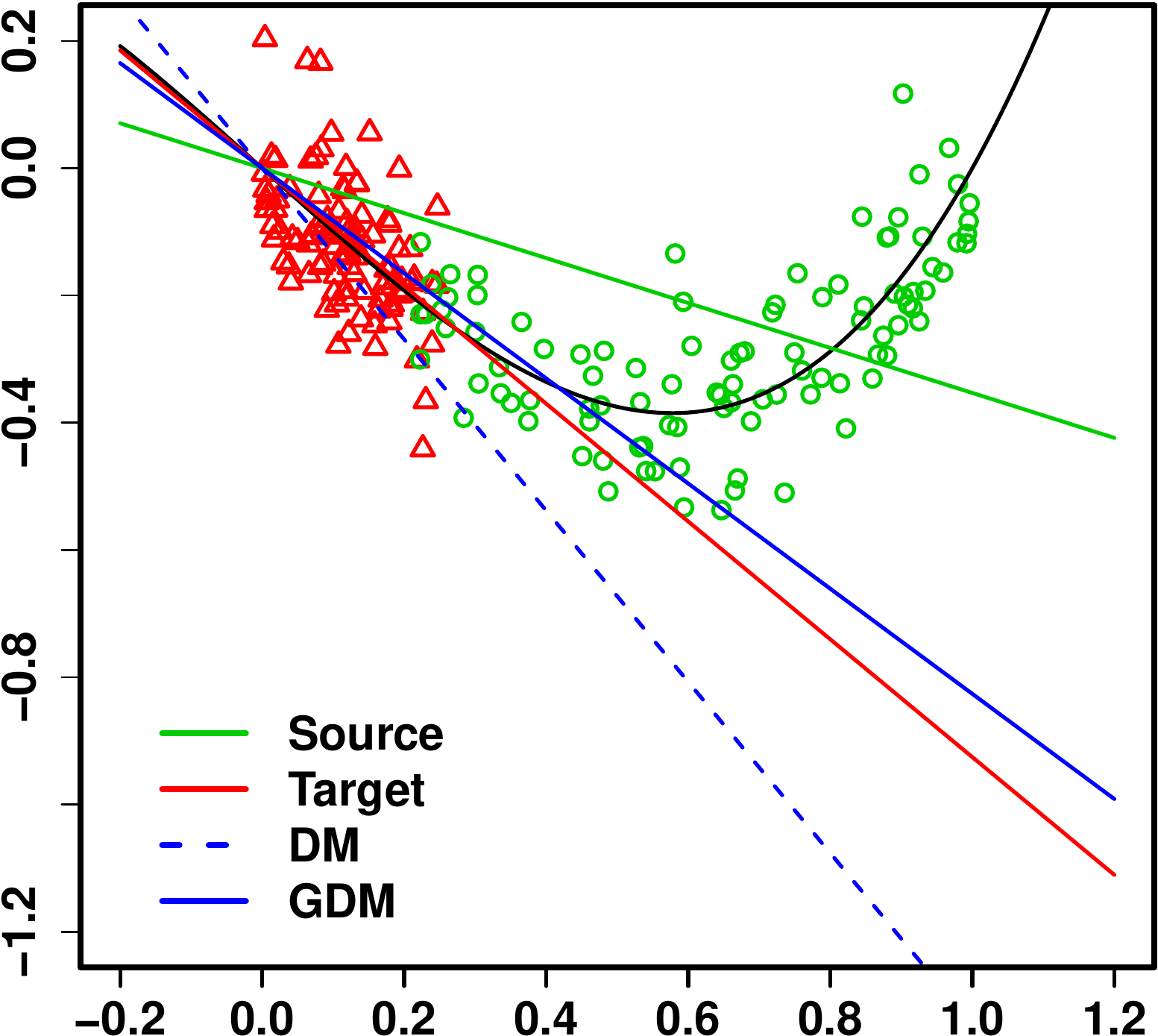} & 
(b) \includegraphics[scale=.32]{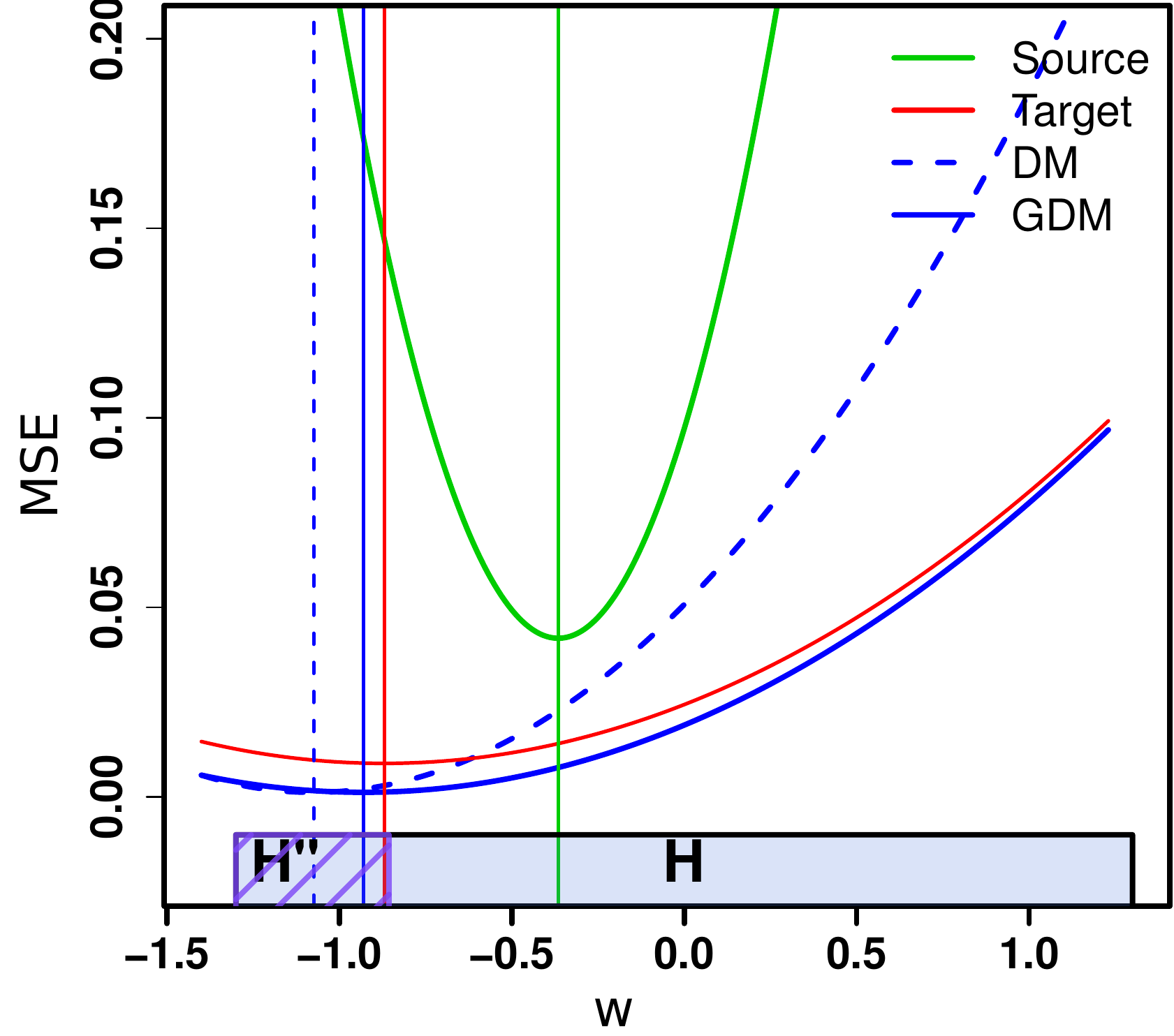}
 \end{tabular}
\caption{(a) Linear hypotheses obtained by training on the source (green circles),
  target (red triangles) and by using the DM (solid blue) and GDM
  algorithms (dashed blue).
 (b)Objective functions associated with training on the source
  distribution, target distribution as well as the GDM and DM
  algorithm. The hypothesis set $H$ and surrogate hypothesis set $H''$
  are shown at the bottom of the plot.}
\label{fig:classifiers}
\end{figure}

\subsection{Synthetic data set} 
In order to illustrate the differences between the GDM and DM
algorithms we generate the following synthetic task which is similar
to the one considered by
\cite{HuangSmolaGrettonBorgwardtScholkopf2006}: source distribution
examples are sampled from the uniform distribution over the interval
$[.2, 1]$ and target data is sampled uniformly over $[0, .25]$. The
labels are given by the map $x \mapsto -x + x^3 + \xi$ where $\xi$ is
a Gaussian random variable with mean $0$ and standard deviation
$0.1$. As hypothesis set we use linear functions without an
offset. Figure~\ref{fig:classifiers}(a) shows the regression
hypotheses obtained by training the DM and GDM algorithm as well as
training on the source and target distributions. The ideal hypothesis
is shown on red.  Notice how the GDM solution approaches the ideal
solution better than DM.  In order to better understand the difference
in the solutions of these algorithms Figure~\ref{fig:classifiers}(b)
depicts the objective function minimized by each algorithm as a
function of the slope $w$ of the linear function, the only variable of
the hypothesis. The vertical lines show the value of the minimizing
hypothesis for each loss. Keeping in mind that the regularization
parameter $\lambda$ used in ridge regression corresponds to a Lagrange
multiplier for the constraint $w^2 \leq \Lambda^2$ for some $\Lambda$
\citep{CortesMohri2013} [Lemma 1], the hypothesis set $H = \{ w | |w|
\leq \Lambda\}$ is depicted at the bottom of this plot. The shaded
region represents the set $H'' = H \cap \{h'' |\cL_{\qmin} (h'') \leq
r \}$. It is clear from this plot that DM helps approximate the target
loss function. Nevertheless, only GDM seems to uniformly approach
it. This should come as no surprise since our algorithm was designed
precisely for this purpose.

\subsection{Adaptation data sets}
We now present the results of  evaluating our algorithm against
several other adaptation algorithms. GDM is compared against
DM and training on the uniform distribution. The following baselines
were also considered:

\begin{enumerate}
\itemsep 0em
 \item The KMM algorithm,  which reweights examples
from the source distribution in an attempt to match the mean of the
source and target data in a feature space induced by a universal
kernel. The hyper-parameters of this algorithm were set to the
recommended values of $B = 1000$ and $\epsilon = \frac{\sqrt{m}}{\sqrt{m} - 1}$.
\item KLIEP. This
algorithm attempts to estimate the importance ratio of the source and
target distribution by modeling this ratio as a mixture of basis
functions and learning the mixture coefficients from the
data. Gaussian kernels were used as basis function where the 
bandwidth for the kernel was selected to be the best performer on the
\emph{test} set. 
\item FE.  This simple algorithm
maps source and target data into a common high-dimensional feature
space where the difference of the distributions is expected to
reduce. 
\end{enumerate}

Unless explicitly stated, our hypothesis set will be a subset of the
RKHS induced by a Gaussian kernel. The learning algorithm used for all
tasks will be kernel ridge regression and the reported risk will be the
mean square error. We follow the setup of
\cite{CortesMohri2011} and select regularization parameter $\lambda$
and Gaussian kernel bandwidth $\sigma$  via 10-fold cross
validation over the training data by doing a grid search for $\lambda
\in \{2^{-25}, \ldots, 2^{-5} \}$ and $\sigma \in \{k d| k = 2^{-10},
\ldots, 1 \}$ where $d$ is the dimensionality of the data.  Finally,
in view of Section~\ref{sec:guarantees}, the surrogate set $H''$ was
selected from the family $\mathscr{H} := \{H'' | H'' = \{h'' |
\cL_{\qmin}(h'') \leq r \vee \cL_{\h Q}(h'') \leq r \}, r \in [0,
\frac{1}{m} \sum_{i=1}^m{y_i^2}] \}$ through validation on a small
amount of data from the target distribution. For our comparisons
to be fair, all algorithms were allowed to use the small amount of
labeled data too. Since, with exception of FE, all other baselines do
not propose a way of dealing with labeled data from the target
distribution, we simply added this data to the training set and ran
the algorithms on the extended source data. 

The first task we consider is given by the 4 {\tt kin-8xy} Delve data sets
\citep{Delve}. These
data sets are all variations of the same model: a realistic simulation
of the forward dynamics of an 8 link all-revolute robot arm. The task
in all data sets is to predict the distance of the end-effector from a
target.  The data sets differ by the degree of non-linearity (fairly
linear , {\tts x=f}, or non-linear, {\tts x=n}) and the amount of
noise in the output (moderate, {\tts y=m} or high, {\tts y=h}). The
data set defines 4 different domains, that is 12 pairs of different
distributions and labeling functions. A sample of 200 points from each
domain was used and 10 labeled points from the target distribution
were used to select $H''$. The experiment was carried out 10 times and
the results of testing on a sample of $400$ points from the target
domain are reported in Figure~\ref{fig:kin}. The bars represent the
median performance of each algorithm. The error bars are the low and
high quartiles respectively. All results are normalized in such a way
that the median performance of training on the target is equal to
1. Since the source labeling function for this task is 
fairly linear, our hypotheses consist of vectors $\mat{w} \in \Rset^8$.
Notice that the performance of all algorithms is comparable when
adapting to {\tt kin8-fm} since both labeling functions are fairly
linear, yet only GDM is able to reasonably adapt to the two data sets
with different labeling functions.

\begin{figure}[t]
\centering
\includegraphics[height=2.5in, width=3.2in]{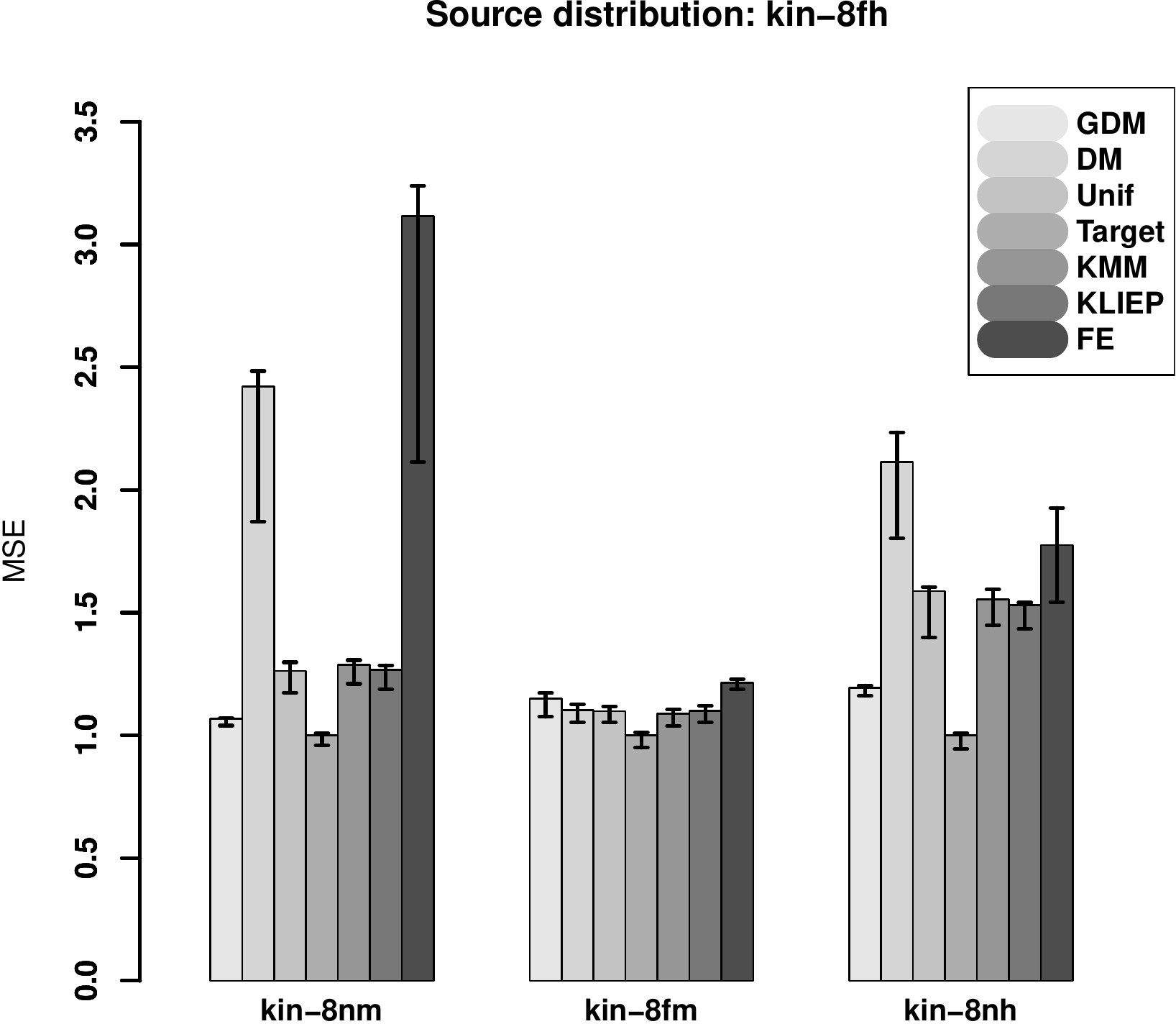}
\caption{MSE performance for different adaptation algorithms when
  adapting from {\tt kin-8fh} to the three other {\tt kin-8xy} domains.}
\label{fig:kin}
\end{figure}

\begin{figure}[t]
\begin{tabular}{cc}
(a) \includegraphics[height=2in, width=2.9in]{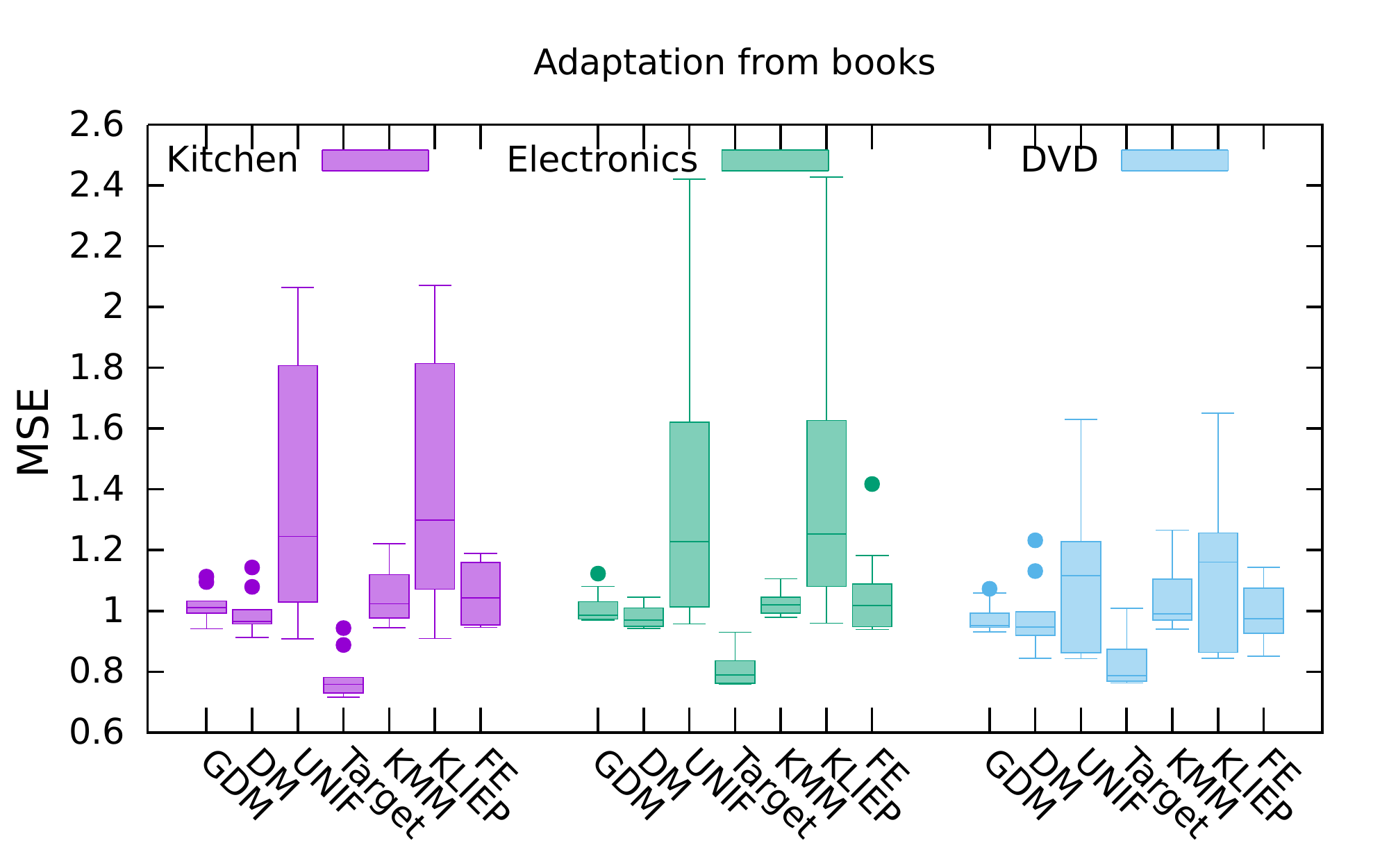} &
(b)\includegraphics[height=2in,width=2.9in]{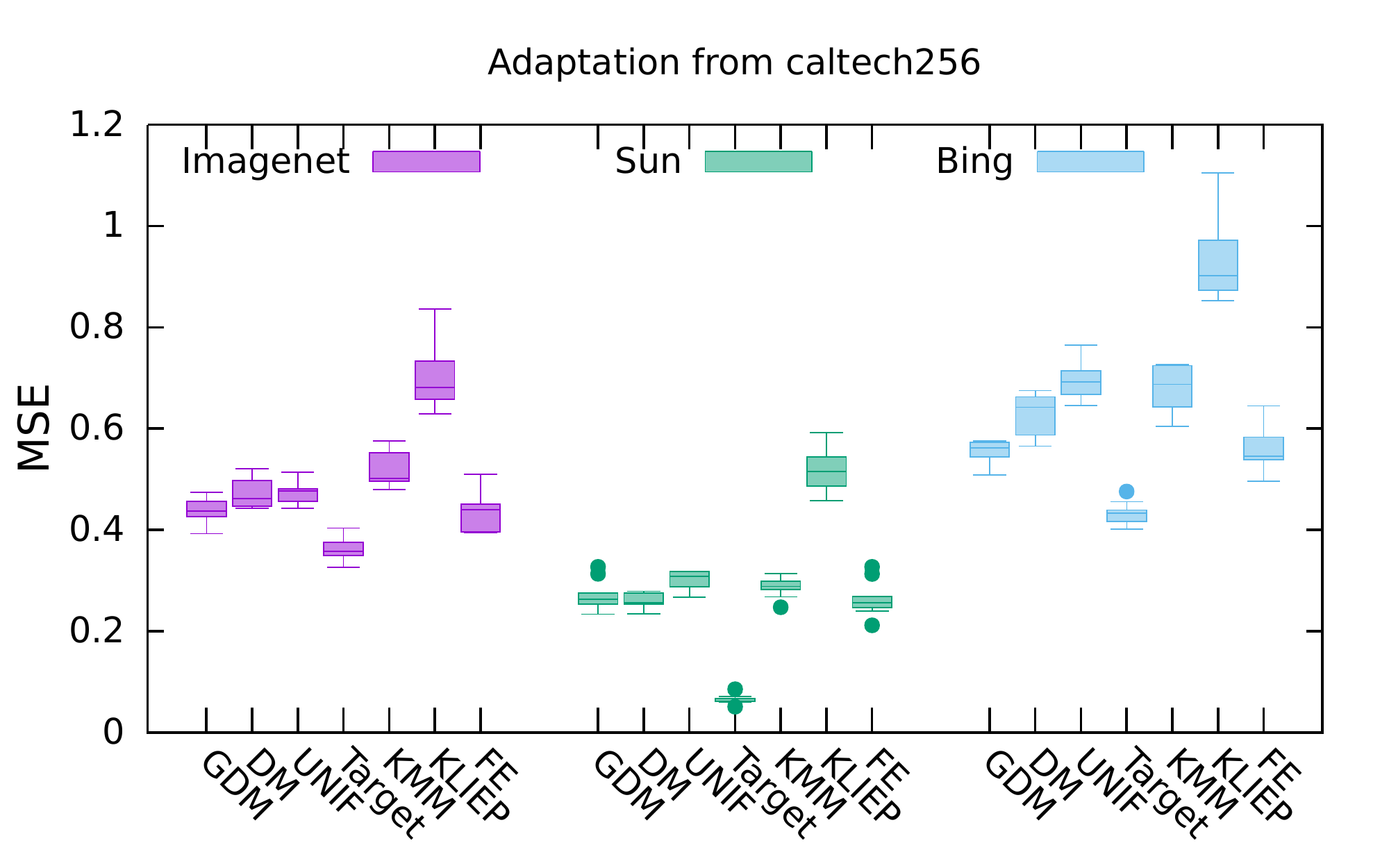}
\end{tabular}
\caption{(a) Performance for the sentiment adaptation task from the {\tt books} domain to all others.
(b) MSE of different algorithms adapting from the {\tt
caltech256} data set to all others.}
\label{fig:realworldsets}
\end{figure}

For our next experiment we consider the cross-domain sentiment
analysis data set of \cite{Blitzer07Biographies}. This data set consists of
consumer reviews from 4 different domains: {\tt books, kitchen, electronics}
and {\tt dvds}. We used the top 5000 unigrams and bigrams as the
features for this task. For each pair of  adaptation tasks we sample $700$
points from the source distribution and $700$ unlabeled points from the
target. Only $50$ labeled  points from the target distribution are
used to tune the parameter $r$ of our
algorithm. The final evaluation is done on a test set of $1000$
points. Figure~\ref{fig:realworldsets}(a) shows MSE of all
algorithms when adapting from {\tt books} to all other domains.

Finally, we consider a novel domain adaptation task \citep{TommasiTC14}
paramount in the computer vision
community. The domains correspond to 4 well known
collections of images: {\tt bing, caltech256, sun} and {\tt
  imagenet}. These data sets have been standardized so that they all share
the same feature representation and labeling
function \citep{TommasiTC14}. We use the data from the first 5 shared classes
and sample 800 labeled points from the source distribution and 800 unlabeled
points from the target distribution as well as 50 labeled target
points to be used for validation of $r$. The results of testing on
$1000$ points from the target domain are depicted in
Figure~\ref{fig:realworldsets}(b) where we trained on {\tt
  caltech256}. The results of all possible adaptation problems for the
sentiment task as well as for the image task are shown in
Appendix~\ref{app:experiments}. The results of this section show that
GDM was the only algorithm that could consistently perform better than
or on par with the DM algorithm, and it consistently outperforms other
algorithms. 

\section{Conclusion}

We presented a new theoretically well-founded domain adaptation
algorithm seeking to make the empirical loss closer to an ideal one
for \emph{each} hypothesis.  This departs from the existing paradigm
of a fixed reweighting for the training losses and leads to a new
theoretical analysis of adaptation. We presented both an SDP solution
for a specific convex set and a more general sampling-based QP
solution for solving the corresponding optimization problem. Our
empirical results show that our algorithm significantly outperforms
the state-of-the art DM algorithm.

\newpage
\bibliography{da}
\newpage
\appendix

\section{Supplementary material}

Here we set the value of $\Lambda$ that will define our hypothesis set
H. $\Lambda = \sqrt{\frac{\mu R}{\lambda}}$, that is $H =\set{h \in
\Hset \colon \| h \|_K \leq \sqrt{\frac{\mu R}{\lambda}}}$.  This does
not impose any additional constraint to the minimization
\eqref{eq:Pmin} as shown by the following lemma.

\begin{lemma}
  Let $h \in \Hset$ a be a solution of the minimization \eqref{eq:Pmin}
  for some training sample $\cS$, then $h$ satisfies the inequality $\|
  h \|_K \leq \sqrt{\frac{\mu R}{\lambda}}$, where $K(x, x) \leq R$ for
  all $x \in \cX$.
\end{lemma}

\begin{proof}
  Since $0$ is an element of $\Hset$, the value of the objective
  function for the minimizer $h$ is upper bounded by the one for $0$:
\begin{equation}
\frac{1}{m} \sum_{i = 1}^m L(h(x_i), y_i) + \lambda \| h \|_K^2
\leq \frac{1}{m} \sum_{i = 1}^m L(0, y_i).
\end{equation}
By the $\mu$-admissibility of the loss, we can then write
\begin{align*}
  \lambda \| h \|_K^2 
\leq \frac{1}{m} \sum_{i = 1}^m L(0, y_i) - \frac{1}{m} \sum_{i = 1}^m L(h(x_i), y_i)
\leq \frac{\mu}{m} \sum_{i = 1}^m |0 - h(x_i)|
\leq \frac{\mu}{m} \sum_{i = 1}^m \| h \|_K K(x_i, x_i) 
\leq \mu R \| h \|_K,
\end{align*}
which implies $\lambda \| h \|_K \leq \mu R$ and concludes the proof.
\end{proof}

\section{SDP formulation}
\label{app:sdpdual}

\begin{replemma}{lemma:sdpdual}
The Lagrangian dual of the problem
\begin{align}
\label{eq:maxprob2}
  \max_{\substack{\a \in \Rset^m \\ \|\Ks \a - \y\|^2 \leq r^2}} 
& \ \frac{1}{2}\|\Kst \a\|^2 - \b^\top \Kt \Kst \b,
\end{align}
is given by
\begin{align*}
\min_{\eta \geq 0, \gamma} & \ \gamma \\  
\text{s. t.} & \ \left( 
\def\arraystretch{1.3}
\begin{array}{cc} 
 -\frac{1}{2} \Kst^\top \Kst + \eta \Ks^2  
& \frac{1}{2}\Kst^\top \Kt \b - \eta \Ks\y  \\
\frac{1}{2} \b^\top  \Kt \Kst - \eta \y^\top \Ks  
&  \eta (\|\y\|^2 - r^2) + \gamma 
\end{array}
\right) \succeq  0. 
\end{align*}
Furthermore, the duality gap for these problems is zero.
\end{replemma}

\begin{proof}
For $\eta \geq 0$ the Lagrangian of \eqref{eq:maxprob2} is given by
\begin{align*}
  L(\a, \eta) & = \frac{1}{2}\|\Kst \a\|^2 - \b^\top \Kt \Kst \a 
- \eta( \|\Ks \a - \y\|^2 - r^2) \\
& =\a^\top \Big(\frac{1}{2} \Kst^\top \Kst - \eta \Ks^2 \Big)
  \a +   (2 \eta \Ks \y - \Kst^\top \Kt \b )^\top  \a - \eta (\|\y\|^2
  - r^2).
\end{align*}
Since the Lagrangian is a quadratic function of $\a$ and that the
conjugate function of a quadratic can be expressed in terms of 
the pseudo-inverse, 
the dual is given by
\begin{align*}
  \min_{\eta \geq 0} & \ \frac{1}{4}(2 \eta \Ks \y - \Kst^\top \Kt
 \b)^\top \Big(\eta \Ks^2 - \frac{1}{2} \Kst^\top \Kst
 \Big)^{\dag}(2 \eta \Ks  \y - \Kst^\top \Kt \b) - \eta(\|\y\|^2 - r^2)\\
\text{s. t. } & \ \eta \Ks^2 - \frac{1}{2} \Kst^\top \Kst\succeq 0.
\end{align*}
Introducing the variable $\gamma$ to replace the objective
function yields the equivalent problem
\begin{flalign*}
\min_{\eta \geq 0, \gamma} & \gamma \\
\text{s. t. } & \ \eta \Ks^2 - \frac{1}{2} \Kst^\top \Kst \succeq
0 \\
& \gamma - \frac{1}{4} (2 \eta \Ks \y - \Kst^\top \Kt
  \b)^\top \Big(\eta \Ks^2 - \frac{1}{2}\Kst^\top \Kst\Big)^{\dag}
(2 \eta \Ks  \y - \Kst^\top \Kt \b) + \eta(\|\y\|^2 - r^2) \geq 0\\
\end{flalign*}
Finally, by the properties of the Schur complement
\citep{BoydVandenberghe2004}, the two constraints above are equivalent
to
\begin{equation*}
\left(
  \begin{array}{cc}
  -\frac{1}{2} \Kst^\top \Kst  + \eta \Ks^2 
& \frac{1}{2} \Kst^\top  \Kt \b - \eta\Ks \y \\
\Big(\frac{1}{2} \Kst^\top  \Kt \b - \eta\Ks \y\Big)^\top 
& \eta (\|\y\|^2 - r) + \gamma
\end{array}
\right) \succeq 0.
\end{equation*}
Since duality holds for a general QCQP with only one constraint
\citep{BoydVandenberghe2004}[Appendix B], the duality gap between
these problems is $0$.
\end{proof}

\begin{repproposition}{prop:cone}
The optimization problem \eqref{eq:kmaxmin} is equivalent to the
following SDP:
\begin{align*}
\max_{\alpha, \beta, \nu, \Z, \z} & \ \frac{1}{2} \Tr(\Kst^\top \Kst \Z)
- \beta - \alpha\\
\text{s. t} & \ \left(
\def\arraystretch{1.3}
\begin{array}{cc}
\nu \Ks^2 + \frac{1}{2}\Kst^\top \Kst - \frac{1}{4} \wt{\mat K}
& \nu  \Ks \y + \frac{1}{4}\wt{\mat K} \z \\
\nu  \y^\top \Ks + \frac{1}{4} \z^\top \wt{\mat K}
&\alpha + \nu (\|\y\|^2 - r^2)
\end{array}
\right) \succeq 0 \quad \wedge \quad
\left(
\begin{array}{cc}
\Z & \z \\
\z^\top & 1
\end{array}
\right) \succeq 0 \\
& \ \left(
\def\arraystretch{1.3}
\begin{array}{cc}
\lambda \Kt + \Kt^2 & \frac{1}{2} \Kt \Kst \z \\
\frac{1}{2} \z^\top \Kst^\top \Kt & \beta
\end{array}
\right) \succeq 0
\quad \wedge \quad \Tr(\Ks^2 \Z) - 2\y^\top \Ks \z + \|\y\|^2 \leq r^2
\quad \wedge \quad \nu \geq 0,
\end{align*}
where $\wt{\mat K} = \Kst^\top \Kt (\lambda \Kt + \Kt^2)^\dag \Kt \Kst$.
\end{repproposition}

\begin{proof}
By Lemma~\ref{lemma:sdpdual}, we may rewrite
\eqref{eq:kmaxmin} as

\begin{align}
\label{eq:firstequiv}
  \min_{\a, \gamma , \eta, \b}
 & \ \b^\top (\lambda \Kt + \Kt^2) \b + \frac{1}{2}
   \a^\top   \Kst^\top \Kst \a - \a^\top  \Kst^\top \Kt \b + \gamma \\
\text{s. t. } & \
 \left(
\def\arraystretch{1.3} 
\begin{array}{cc} 
 -\frac{1}{2} \Kst^\top \Kst + \eta \Ks^2  &
 \frac{1}{2}\Kst^\top \Kt \b - \eta \Ks \y \\
\frac{1}{2} \b^\top  \Kt \Kst - \eta \y^\top \Ks &
 \eta (\|\y\|^2 - r^2) + \gamma 
\end{array}
\right)  \quad \wedge \quad \eta \geq 0 \nonumber \\
& \ \|\Ks \a - \y\|^2 \leq r^2. \nonumber
\end{align}
Let us apply the change of variables $\b = \frac{1}{2}(\lambda \Kt +
\Kt^2)^{\dag} \Kt \Kst \a +\v$. The following equalities can be easily
verified.
\begin{align*}
\b^\top (\lambda \Kt + \Kt^2) \b
&  = \frac{1}{4} \a^\top \Kst^\top \Kt (\lambda \Kt + \Kt^2)^\dag \Kt \Kst \a + \v^\top \Kt \Kst \a + \v^\top (\lambda \Kt + \Kt^2) \v.\\
\a^\top \Kst^\top  \Kt \b
& = \frac{1}{2} \a^\top \Kst^\top \Kt (\lambda \Kt + \Kt^2)^\dag \Kt \Kst \a + \v^\top \Kt \Kst \a.
\end{align*}
Thus, replacing $\b$ on \eqref{eq:firstequiv} yields
\begin{align*}
  \min_{\a, \v, \gamma , \eta}
 & \ \v^\top (\lambda \Kt + \Kt^2) \v + \
   \a^\top \Big(\frac{1}{2} \Kst^\top \Kst
   - \frac{1}{4}\wt{\mat K} \Big)\a + \gamma\\
\text{s. t. } & \
 \left(
\def\arraystretch{1.3} 
\begin{array}{cc} 
 -\frac{1}{2} \Kst^\top \Kst + \eta \Ks^2 
& \frac{1}{4} \wt{\mat K}\a + \frac{1}{2}\Kst^\top \Kt \v - \eta \Ks \y \\
\frac{1}{4} \a^\top \wt{\mat K} + \frac{1}{2}\v^\top \Kt \Kst 
- \eta \y^\top \Ks
& \eta (\|\y\|^2 - r^2) + \gamma 
\end{array}
\right) \succeq 0  \quad \wedge \quad \eta \geq 0 \nonumber \\
& \ \|\Ks \a - \y\|^2 \leq r^2. \nonumber
\end{align*}
Introducing the scalar multipliers $\mu, \nu \geq 0$ and the matrix 
\begin{equation*}
\left(
\begin{array}{cc}
\Z & \z \\
\z^\top & \wt z,
\end{array}
\right) \succeq 0
\end{equation*}
as a multiplier for the matrix constraint, we can form the Lagrangian:
\begin{multline*}
\mathfrak{L} := \v^\top (\lambda \Kt + \Kt^2) \v 
+ \a^\top \Big(\frac{1}{2} \Kst^\top \Kst - \frac{1}{4}\wt{\mat K} \Big)\a
+ \gamma - \mu \eta + \nu (\|\Ks \a - \y\|^2 - r^2) \\
- \Tr  \left( \left(
\begin{array}{cc}
\Z & \z \\
\z & \wt z
\end{array}
\right)
\left( 
\def\arraystretch{1.3}
\begin{array}{cc} 
 -\frac{1}{2} \Kst^\top \Kst + \eta \Ks^2 
& \frac{1}{4} \wt{\mat K}\a + \frac{1}{2}\Kst^\top \Kt \v - \eta \Ks \y \\
\frac{1}{4} \a^\top \wt{\mat K} + \frac{1}{2}\v^\top \Kt \Kst 
- \eta \y^\top \Ks
& \eta (\|\y\|^2 - r^2) + \gamma 
\end{array}
\right) \right). 
\end{multline*}
The KKT conditions $\frac{\partial \mathfrak L}{\partial \eta} =
\frac{\partial \mathfrak L}{\partial \gamma} = 0$ trivially imply $\wt
z= 1$ and $\Tr(\Ks^2 \Z) - 2\y^\top \Ks \z + \|\y\|^2 - r^2 + \mu = 0$.
These constraints on the dual variables guarantee that the primal
variables $\eta$ and $\gamma$ will vanish from the Lagrangian, thus
yielding
\begin{multline*}
\mathfrak{L} = \frac{1}{2} \Tr(\Ks^2 \Z) + \nu(\|\y\|^2 - r^2)
+ \v^\top (\lambda \Kt + \Kt^2) \v^\top - \z^\top \Kst^\top \Kt \v\\
+ \a^\top\Big(\nu \Ks^2 + \frac{1}{2}\Kst^\top \Kst - \frac{1}{4} \wt{\mat K}\Big) \a -\Big(2 \nu  \Ks \y + \frac{1}{2}\wt{\mat K} \z\Big)^\top \a. 
\end{multline*}
This is a quadratic function on the primal variables $\a$ and $\v$
with minimizing solutions
\begin{equation*}
\a = \frac{1}{2} \Big(\nu \Ks^2 + \frac{1}{2}\Kst^\top \Kst - \frac{1}{4} \wt{\mat K}\Big)^\dag \Big(2 \nu  \Ks \y + \frac{1}{2}\wt{\mat K} \z\Big) 
\qquad \text{and} \qquad
\v = \frac{1}{2}(\lambda \Kt +  \Kt^2)^{\dag}\Kt \Kst \z,
\end{equation*}
and optimal value equal to the objective of the Lagrangian dual:
\begin{multline*}
 \frac{1}{2} \Tr(\Kst^\top \Kst \Z) + \nu(\|\y\|^2 - r^2)
- \frac{1}{4} \z^\top \wt{\mat K} \z  \\
- \frac{1}{4} \Big(2 \nu  \Ks \y + \frac{1}{2}\wt{\mat K} \z\Big)^\top
\Big(\nu \Ks^2 + \frac{1}{2}\Kst^\top \Kst - \frac{1}{4} \wt{\mat K}\Big)^\dag
\Big(2 \nu  \Ks \y + \frac{1}{2}\wt{\mat K} \z\Big).
\end{multline*}
As in Lemma~\ref{lemma:sdpdual}, we apply the properties of the Schur
complement to show that the dual is given by
\begin{align*}
\max_{\alpha, \beta, \nu, \Z, \z} & \ \frac{1}{2} \Tr(\Kst^\top \Kst \Z)
- \beta - \alpha\\
\text{s. t} & \ \left(
\def\arraystretch{1.3}
\begin{array}{cc}
\nu \Ks^2 + \frac{1}{2}\Kst^\top \Kst - \frac{1}{4} \wt{\mat K}
& \nu  \Ks \y + \frac{1}{4}\wt{\mat K} \z\  \\
\nu \y^\top  \Ks + \frac{1}{4} \z^\top \wt{\mat K}
&\alpha + \nu (\|\y\|^2 - r^2)
\end{array}
\right) \succeq 0 \quad \wedge \quad
\left(
\begin{array}{cc}
\Z & \z \\
\z^\top & 1
\end{array}
\right) \succeq 0 \\
& \ \Tr(\Ks^2 \Z) - 2\y^\top \Ks \z + \|\y\|^2 \leq r^2
\quad \wedge \quad \beta \geq \frac{1}{4} \z^\top \wt{\mat K} \z 
\quad \wedge \quad \nu \geq 0
\end{align*}
Finally, recalling the definition of $\wt{\mat K}$ and using the
Schur complement one more time we arrive to the final SDP formulation:
\begin{align*}
\max_{\alpha, \beta, \nu, \Z, \z} & \ \frac{1}{2} \Tr(\Kst^\top \Kst \Z)
- \beta - \alpha\\
\text{s. t} & \ \left(
\def\arraystretch{1.3}
\begin{array}{cc}
\nu \Ks^2 + \frac{1}{2}\Kst^\top \Kst - \frac{1}{4} \wt{\mat K}
& \nu  \Ks \y + \frac{1}{4}\wt{\mat K} \z \\
\nu  \y^\top \Ks + \frac{1}{4} \z^\top \wt{\mat K}
&\alpha + \nu (\|\y\|^2 - r^2)
\end{array}
\right) \succeq 0 \quad \wedge \quad
\left(
\begin{array}{cc}
\Z & \z \\
\z^\top & 1
\end{array}
\right) \succeq 0 \\
& \ \left(
\def\arraystretch{1.3}
\begin{array}{cc}
\lambda \Kt + \Kt^2 & \frac{1}{2} \Kt \Kst \z \\
\frac{1}{2} \z^\top \Kst^\top \Kt & \beta
\end{array}
\right) \succeq 0
\quad \wedge \quad \Tr(\Ks^2 \Z) - 2\y^\top \Ks \z + \|\y\|^2 \leq r^2
\quad \wedge \quad \nu \geq 0.
\end{align*}
\end{proof}

\section{QP  formulation}
\label{app:qpformula}

\begin{repproposition}{prop:dual}
Let $\mat Y =(Y_{ij}) \in \Rset^{n \times k}$ be the
matrix defined by $Y_{ij} = n^{-1/2} h_j(x_i')$ and $\y' = (y'_1,
\ldots, y'_k)^\top \in \Rset^k $ the vector defined by $y'_i = n^{-1}
\sum_{j=1}^n h_i(x'_j)^2$. Then, the dual problem of
\eqref{eq:optmaxapp} is given by
\begin{align}
\label{eq:dual}
\max_{\bm \alpha, \bm \gamma, \beta} & \ -\Big(\mat Y \bm \alpha + \frac{\bm
  \gamma}{2} \Big)^\top \Kt\Big(\lambda \I + \frac{1}{2}\Kt\Big)^{-1}
\Big(\mat Y \bm \alpha  + \frac{\bm \gamma}{2}\Big) - \frac{1}{2} \bm
\gamma^\top \Kt \Kt^\dag \bm \gamma +  \bm \alpha^\top \y' - \beta\\
\text{s.t.} & \ \1^\top \bm \alpha = \frac{1}{2}, \qquad  \1
\beta \geq -\mat Y^\top \bm \gamma, \qquad \bm \alpha\geq
0, \nonumber
\end{align}
where $\1$ is the vector in $\Rset^k$ with all components equal to
$1$. Furthermore, the solution $h$ of \eqref{eq:optmaxapp} can be
recovered from a solution $(\bm \alpha, \bm \gamma, \beta)$ of
\eqref{eq:dual} by $\forall x, h(x) =\sum_{i = 1}^n a_i K(x_i, x)$,
where $\bm a = \big(\lambda \I + \frac{1}{2}\Kt)^{-1}(\mat Y \bm
\alpha + \frac{1}{2}\bm \gamma)$.
\end{repproposition}

We will first prove a simplified version of the proposition for the
case of linear hypotheses, i.e. we can represent hypotheses in $\Hset$
and elements of $\cX$ as vectors $\w, \x \in \Rset^d$
respectively. Define $\X' = n^{-1/2} (\x_1', \ldots, \x_n')$ to be the
matrix whose columns are the normalized sample points from the target
distribution. Let also $\{\w_1, \ldots, \w_k\}$ be a sample taken from
$\partial H''$ and define $\bm W := (\w_1, \ldots, \w_k) \in \Rset^{d
  \times k}$. Under this notation, problem \eqref{eq:optmaxapp} may be
rewritten as
\begin{equation}
\label{eq:linoptimization}
\min_{\w \in \Rset^d} \lambda \|\w\|^2 + \frac{1}{2} \max_{i =1,
  \dots, k} \|\X'^\top(\w - \w_i)\|^2 + \frac{1}{2} \min_{\w' \in
  \mathcal C} \|\X'^\top (\w - \w')\|^2
\end{equation}

\begin{lemma}
The Lagrange dual of problem \eqref{eq:linoptimization} is given by
\begin{align*}
\max_{\bm \alpha, \bm \gamma, \beta}& \  -\Big(\mat Y \bm \alpha + \frac{\bm \gamma}{2}\Big)^\top
\X'^\top \Big(\lambda \I + \frac{\X' \X'^\top}{2}\Big)^{-1}\X' \Big(\bm
Y \bm\alpha + \frac{\bm \gamma}{2}\Big) - \frac{1}{2}\bm \gamma^\top \X'^\top (\X' \X'^\top)^\dag
\X' \bm \gamma + \bm \alpha^\top \y' - \beta\\
\text{s. t.} & \ \1^\top \bm \alpha = \frac{1}{2} \quad \quad  \1
\beta \geq- \mat Y^\top \bm \gamma \quad \quad \bm \alpha\geq
0,
\end{align*}
where $\mat Y = \X'^\top \bm W$ and $\y'_i =\|\X'^\top \w_i\|^2.$
\end{lemma}

\begin{proof}
By applying the change of variable $\bu = \w' - \w$, problem
\eqref{eq:linoptimization} is can be made equivalent to 
\begin{equation*}
  \min_{\w \in \Rset^d \bu \in \mathcal{C} - \w} \lambda  \|\w\|^2 +
  \frac{1}{2} \|\X'^\top  \w\|^2 +
  \frac{1}{2}\|\X'^\top u\|^2 + \frac{1}{2}\max_{i=1,\ldots, k} \|\bm
  \X'^\top \w_i\|^2 -2
  \w_i^ \top \X' \X'^\top \w.
\end{equation*}
By making the constraints on $\bu$ explicit and replacing the
maximization term with the variable $r$ the above problem becomes
\begin{align*}
  \min_{\w, \bu , r, \bm \mu}
& \quad \lambda \|\w\|^2 +   \frac{1}{2}
  \|\X'^\top  \w\|^2 +  \frac{1}{2}\|\X'^\top \bu\|^2 + \frac{1}{2}r \\
\text{s. t.}
& \quad \1 r \geq \y' - 2 \mat Y^\top \X'^\top  \w \\
& \quad \1^\top \bm \mu =1 \qquad \bm \mu \geq 0 \qquad \bm W \bm \mu - \w = \bu.
\end{align*}
For $\bm \alpha, \bm \delta \geq 0$, the Lagrangian of this problem is
defined as
\begin{align*}
\mathfrak{L}(\w, \bu, \bm \mu, r, \bm \alpha, \beta, \bm \delta, \bm \gamma') 
& = \lambda \|\w\|^2 +   \frac{1}{2}  \|\X'^\top  \w\|^2 
+ \frac{1}{2}\|\X'^\top \bu\|^2 + \frac{1}{2}r 
+ \bm \alpha^\top(\y' - 2(\X' \mat Y)^\top \w - \1 r) \\
& \mspace{40mu} + \beta(\1^\top  \bm \mu - 1) - \bm \delta^\top \bm \mu 
+ \bm \gamma'^\top(\bm W \bm  \mu - \w - \bu).
\end{align*}
Minimizing with respect to the primal variables yields the following
KKT conditions:
\begin{align}
  \1^\top \bm \alpha = \frac{1}{2} & \quad \quad
  \1\beta  = \bm \delta -\bm W^\top \bm \gamma'. \label{eq:linear}\\
  \bm \X'\X'^\top \bu = \bm \gamma' & \quad \quad
 2 \left(\lambda \I + \frac{\X' \X'^\top}{2}\right)\w = 2(\X' \bm
 Y)\bm \alpha + \bm \gamma' \label{eq:quad}
\end{align}
Condition \eqref{eq:linear} implies that the terms involving $r$
and $\bm \mu$ will vanish from the Lagrangian. Furthermore, the first
equation in \eqref{eq:quad} implies that any feasible $\bm \gamma'$
must satisfy $\bm \gamma' = \X' \bm \gamma$ for some $\gamma \in
\Rset^n$. Finally, it is immediate that $\bm \gamma'^\top \bu =
\bu^\top \X' \X'^\top \bu$ and $2\w ^\top\left(\lambda \I + \frac{\X'
\X'^\top}{2}\right) \w = 2 \bm \alpha^\top (\X' \mat Y)^\top \w \bm
+ \bm \gamma'^\top \w$. Thus, at the optimal point, the Lagrangian
becomes
 \begin{align*}
&  \quad  -\w^\top \Big(\lambda \I + \frac{1}{2} \X' \X'^\top\Big) \w 
- \frac{1}{2} \bu^\top  \X' \X'^\top  \bu + \bm \alpha^\top \y' - \beta \\
\text{s. t.} 
& \quad \1^\top \bm \alpha = \frac{1}{2} \quad \quad  \1
\beta = \bm \delta - \bm W^\top  \bm \gamma' \quad \quad \bm \alpha\geq 0 \wedge
\bm \delta \geq 0.
 \end{align*}
The positivity of $\bm \delta$ implies that $\1\beta \geq -\bm
W^\top \bm \gamma'$. Solving for $\w$ and $\bu$ on \eqref{eq:quad} and
applying the change of variable $\X' \bm \gamma = \bm\gamma'$ we obtain
the final expression for the dual problem:
\begin{align*}
\max_{\bm \alpha, \bm \gamma, \beta}
& \  -\Big(\mat Y \bm \alpha + \frac{\bm \gamma}{2}\Big)^\top
\X'^\top \Big(\lambda \I + \frac{\X' \X'^\top}{2}\Big)^{-1}\X' \Big(\bm
Y \bm\alpha + \frac{\bm \gamma}{2}\Big) - \frac{1}{2}\bm \gamma^\top \X'^\top (\X' \X'^\top)^\dag
\X' \bm \gamma + \bm \alpha^\top \y' - \beta\\
\text{s. t.} 
& \ \1^\top \bm \alpha = \frac{1}{2} \quad \quad  \1
\beta \geq- \mat Y^\top \bm \gamma \quad \quad \bm \alpha\geq 0,
\end{align*}
where we have used the fact that $\mat Y^\top \bm \gamma = \W \X'^\top
\bm \gamma$ to simplify the constraints. Notice also that we can
recover the solution $\w$ of problem \eqref{eq:linoptimization} as $\w
= (\lambda \I + \frac{1}{2} \X'^\top \X')^{-1}\X'( \mat Y \bm \alpha +
\frac{1}{2}\bm \gamma)$
\end{proof}

Using the matrix identities $ \X'(\lambda \I + \X'^\top \X')^{-1} =
(\lambda \I + \X' \X'^\top) \X'$ and $\X'^\top \X'
(\X'^\top \X')^\dag  =  \X'^\top (\X' \X'^\top) ^\dag
\X'$, the proof of Proposition~\ref{prop:dual} is now immediate.
\\

\begin{proof}[Proposition~\ref{prop:dual}] 
We can rewrite the dual objective of the
previous lemma in terms of the Gram matrix $\X'^\top \X'$ alone as follows:
\begin{align*}
\max_{\bm \alpha, \bm \gamma, \beta}
& \  -\Big(\mat Y \bm \alpha + \frac{\bm \gamma}{2}\Big)^\top
\X'^\top \X' \Big(\lambda \I + \frac{\X'^\top \X'}{2}\Big)^{-1} \Big(\bm
Y \bm\alpha + \frac{\bm \gamma}{2}\Big) - \frac{1}{2}\bm \gamma^\top
\X'^\top \X'(\X'^\top \X')^\dag \bm \gamma + \bm \alpha^\top \y' - \beta\\
\text{s. t.}
& \ \1^\top \bm \alpha = \frac{1}{2} \quad \quad  \1
\beta \geq- \mat Y^\top \bm \gamma \quad \quad \bm \alpha\geq 0.
\end{align*}
By replacing $\X'^\top \X'$ by the more general kernel matrix $\bm
\Kt$ (which corresponds to the Gram matrix in the feature space) we
obtain the desired expression for the dual. Additionally, the same
matrix identities applied to condition \eqref{eq:quad} imply that
the optimal hypothesis $h$ is given by $h(x) =\sum_{i=1}^n a_i
K(x_i',x)$ where $\bm a = (\lambda \I + \frac{1}{2}\Kt)^{-1}(\mat Y
\bm \alpha + \frac{\bm \gamma}{2})$.
\end{proof}

\section{$\mu$-admissibility}
\label{app:muadmissible}

\begin{lemma}[Relaxed triangle inequality]
\label{lemma:relaxed}
For any $p \geq 1$, let $L_p$ be the loss defined over $\Rset^N$ by
$L_p(\x, \y) = \| \y - \x \|^p$ for all $\x, \y \in \Rset^N$. Then,
the following inequality holds for all $\x, \y, \z \in \Rset^N$:
\begin{equation*}
L_p(\x, \z) \leq 2^{q -1} [ L_p(\x, \y) + L_p(\y, \z) ].
\end{equation*}
\end{lemma}

\begin{proof}
Observe that
\begin{equation*}
	L_p(\x, \z) = 2^p  \Big\| \frac{\x - \y}{2} + \frac{ \y - \z}{2} \Big\|^p.
\end{equation*}
For $p \geq 1$, $x \mapsto x^p$ is convex, thus,
\begin{equation*}
L_p(\x, \z) \leq 2^p  \frac{1}{2} \Big[ \| (\x - \y) \|^p +
\| (\y - \z) \|^p \Big] 
= 2^{p -  1} [ L_p(\x, \z) + L_p(\y, \z) ],
\end{equation*}
which concludes the proof.
\end{proof}

\begin{lemma}
\label{lemma:muadmissible}
Assume that $L_p(h(x), y) \leq M$ for all $x \in \cX$ and $y \in \cY$,
then $L_p$ is $\mu$-admissible with $\mu = p M^{p - 1}$.
\end{lemma}
\begin{proof}
Since $x \mapsto x^p$ is $p$-Lipschitz over $[0,1]$ we can write
\begin{align*}
| L(h(x),y) - L(h'(x), y) | 
& = M^p \bigg|\Big(\frac{|h(x) - y|}{M}\Big)^p  -
\Big(\frac{|h'(x) - y|}{M}\Big)^p \bigg|\\
& \leq p M^{p-1}|h(x) - y + y - h'(x)|  \\
& = p M^{p-1} |h(x) - h'(x)|,
\end{align*}
which concludes the proof.
\end{proof}

\begin{lemma}
\label{lemma:holder}
Let $L$ be the $L_p$ loss for some $p \geq 1$ and let $h, h', h''$ be
functions satisfying $L_p(h(x), h'(x)) \leq M$ and $L_p(h''(x),
h'(x)) \leq M$ for all $x \in \cX$, for some $M \geq 0$. Then, for
any distribution $\cD$ over $\cX$, the following inequality holds:
\begin{equation}
| \cL_\cD(h, h') - \cL_\cD(h'', h') | \leq p M^{p-1}[\cL_\cD(h, h'')]^{\frac{1}{p}}.
\end{equation}
\end{lemma}

\begin{proof}
Proceeding as in the proof of Lemma~\ref{lemma:muadmissible}, we obtain
\begin{align*}
| \cL_\cD(h, h') - \cL_\cD(h'', h') |
& = | \E_{x \in \cD}\big[L_p(h(x),
h'(x)) - L_p(h''(x), h'(x)\big]  | \\
& \leq p M^{p-1} \E_{x \in \cD} \big[|h(x) - h''(x)| \big].
\end{align*}
Since $p \geq 1$, by Jensen's inequality, we can write $\E_{x \in \cD}
\big[|h(x) - h''(x)| \big] \leq \E_{x \in \cD} \big[|h(x) - h''(x)|^p
\big]^{1/p} = [\cL_\cD(h, h'')]^{\frac{1}{p}}$.
\end{proof}
\section{Experiments}
\label{app:experiments}
Here we report the results of all pairs of adaptation problems for the
image task and the sentiment task. Each row of the plot corresponds
to a different source domain and each column in the plot corresponds
to a different target domain. The results reported here are the mean
performance of the same experiment repeated 10 times. The error bars
represent 1 standard deviation. 
\begin{figure}[t]
\centering
\includegraphics[scale=.9]{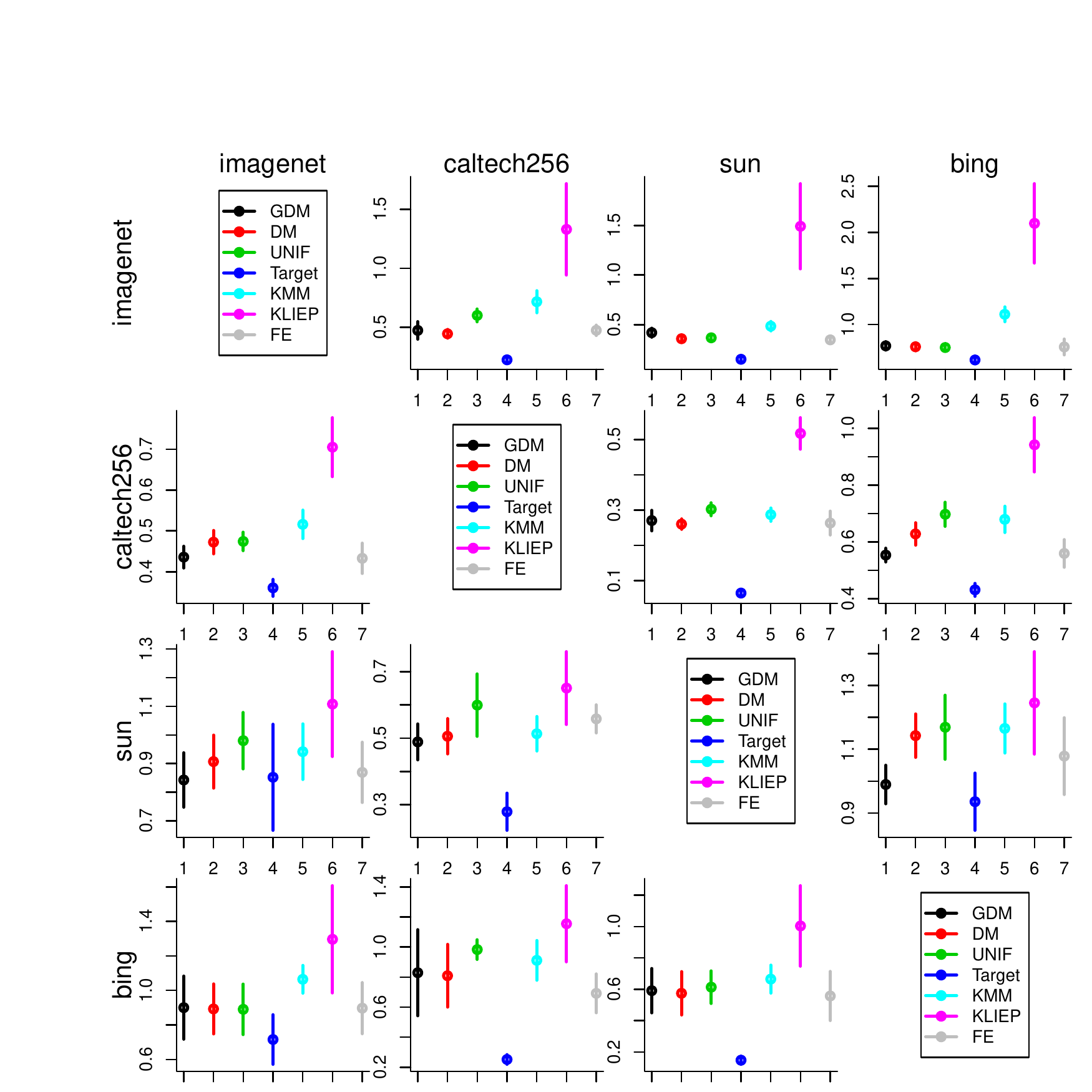}
\caption{Adaptation results for the image data set}
\end{figure}
\begin{figure}[t]
\centering
\includegraphics[scale=.9]{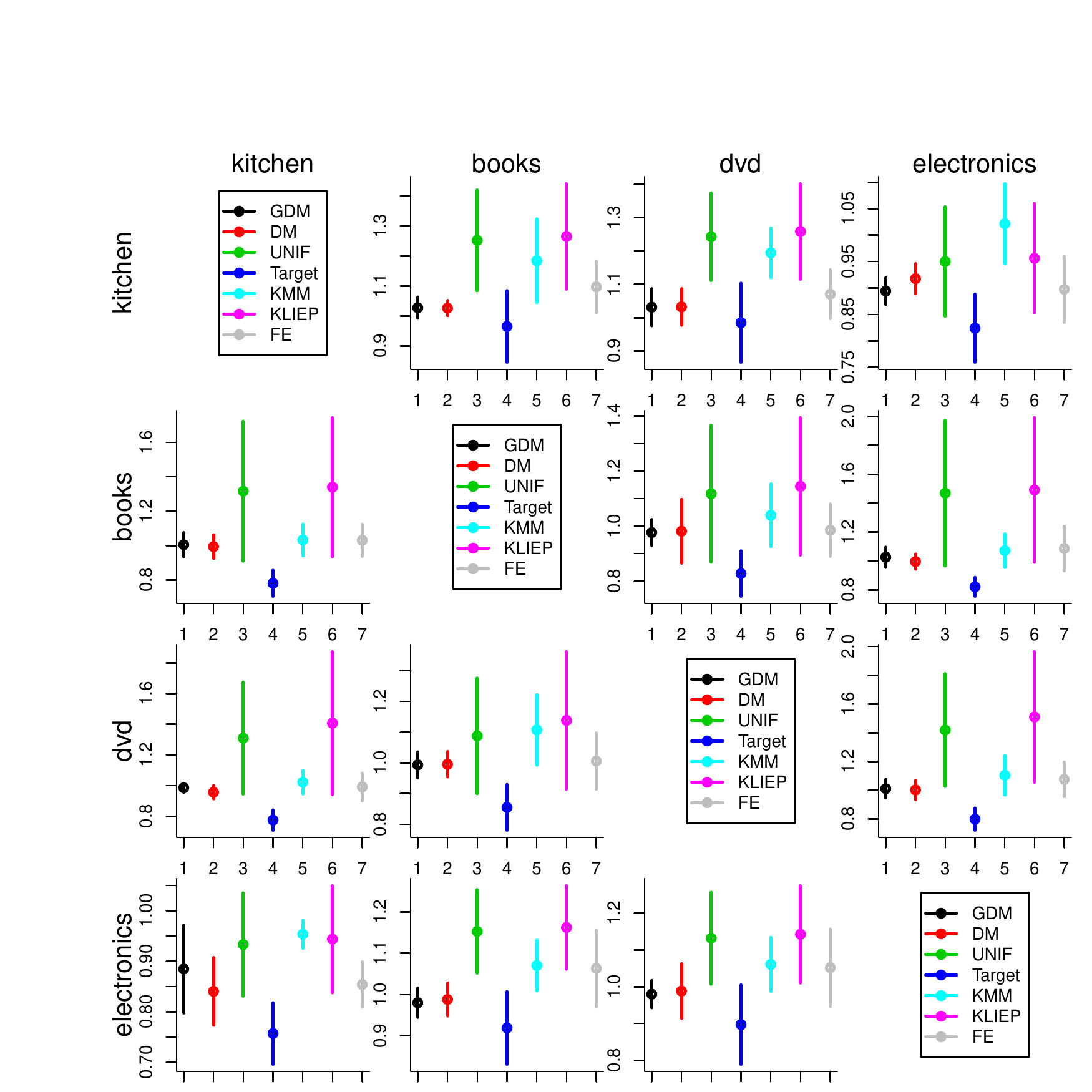}
\caption{Adaption results for the sentiment data set}
\end{figure}

\end{document}